\documentclass[11pt]{article}

\usepackage[final]{acl}

\usepackage{times}
\usepackage{latexsym}

\usepackage[T1]{fontenc}
\usepackage[utf8]{inputenc}

\usepackage{microtype}

\usepackage{inconsolata}

\usepackage{graphicx}

\usepackage{algorithm}
\usepackage{algorithmic}
\usepackage{booktabs}
\usepackage{tabularx}
\usepackage{geometry}
\usepackage[table]{xcolor}
\usepackage{amsmath}
\usepackage{amsthm}
\usepackage{adjustbox}
\usepackage{multirow}
\usepackage{authblk}
\usepackage{hyperref}
\usepackage{chngcntr}
\usepackage{orcidlink}
\usepackage{wrapfig}
\usepackage{amssymb}
\usepackage{cleveref}
\usepackage[dvipsnames]{xcolor}

\newtheorem{proposition}{Proposition}
\newtheorem{lemma}{Lemma}

\AddToHook{cmd/appendix/before}{\crefalias{section}{appendix}}
\usepackage{subcaption}
\usepackage{dblfloatfix}
\usepackage{lipsum}% http://ctan.org/pkg/lipsum
\title{IWP: Token Pruning as Implicit Weight Pruning in Large Vision Language Models}

\author{
 \textbf{Dong-Jae Lee}\thanks{Equal contribution.} \quad
 \textbf{Sunghyun Baek}$^*$ \quad
 \textbf{Junmo Kim} \quad
\\
 \textsuperscript{}KAIST
\\
    \texttt{\{jhtwosun,
    baeksh,
    junmo.kim\}@kaist.ac.kr}
}

\begin{document}
\maketitle
\begin{abstract}
Large Vision–Language Models (LVLMs) show impressive performance across image and video understanding tasks, yet their computational cost grows rapidly with the number of visual tokens. 
% Existing token pruning methods mitigate this issue through empirical approaches while overlooking the internal mechanism of attention.
Existing token pruning methods address this issue by reducing the number of visual tokens, but overlook the internal mechanism of attention.
% Existing token pruning methods overlooks the internal mechanism of attention.
In this paper, we propose a novel training-free token pruning framework grounded in the dual-form perspective of attention. We reformulate attention as an implicit linear layer whose weight matrix is the sum of rank-1 outer products, generated by a single token's key–value pair. 
% Token pruning thus reduces to selecting an optimal subset of these rank-1 updates that best approximates the original dual weight matrix. 
Token pruning thus reduces to identifying an optimal subset of these rank-1 updates that retains the essential information of the original dual weight matrix.
Extending this perspective to standard softmax attention in LVLMs, we derive two novel metrics quantifying a token's information magnitude and information duplication, respectively.
To efficiently select the subset with the proposed metrics, we introduce Progressive Chunked Maximal Marginal Relevance. Extensive experiments demonstrate that our method achieves a better trade-off between performance and efficiency, while providing another perspective on existing pruning approaches. Code is available at: \url{https://github.com/jhtwosun/IWP}.
\end{abstract}

\section{Introduction}
\label{sec:intro}

Recent advances in Large Vision-Language Models (LVLMs) \cite{liu2023visual, li2025llavaonevision, wang2024qwen2, yang2025qwen25} have significantly enhanced multimodal understanding. To achieve fine-grained perception, there has been a notable shift towards processing high-resolution images and extended videos. 
However, this trend inherently generates an excessive number of visual tokens. 
% Since the computational cost of self-attention scales quadratically with sequence length, this surge in visual tokens leads to prohibitive computational and memory overheads, particularly during the prefill phase and KV cache storage. 
Since the computational cost of attention scales quadratically with sequence length, this surge in visual tokens leads to prohibitive computational and memory overheads.
To alleviate this computational bottleneck, token pruning \cite{bolya2023tome}, originally developed for accelerating standalone vision transformers \cite{dosovitskiy2021an}, has emerged as a prominent solution for LVLMs \cite{chen2024fastv, dhouib2025pact, yang2025visionzip, alvar2025divprune, lin2025vtw}.
Existing methods typically regard attention scores as token importance scores \cite{chen2024fastv, yang2025visionzip, dhouib2025pact}, enforce token diversity to minimize redundancy \cite{dhouib2025pact, alvar2025divprune}, or combine both objectives \cite{zhang2025cdpruner, deng2025scope}. Despite their remarkable performance, it remains unclear how token pruning affects the internal mechanism of attention.
% In particular, attention-based methods capture only the query–key alignment while overlooking the contribution of value, and diversity-based methods measure redundancy in the raw feature space rather than in the space where attention actually constructs its output.

In this paper, we aim to establish a principled token pruning framework that departs from empirical approaches.
To this end, we adopt the dual-form perspective of linear attention \cite{irie2022dual, dai2023can, ren2024towards}, which reformulates attention as an implicit linear transformation whose weight matrix is constructed from key–value pairs and applied to the query to produce the output.
Under this perspective, the implicit weight matrix can be decomposed into token-wise update matrices, each represented as the outer product of the corresponding key and value.
% Consequently, token pruning reduces to identifying the subset whose summation best preserves the information of the original dual weight matrix.
Consequently, token pruning reduces to identifying the subset whose summation maintains the essential information of the original dual weight.

Building on this perspective, we extend the dual-form analysis to standard softmax attention in LVLMs and derive a novel, training-free token selection algorithm that captures both information magnitude and information redundancy. Specifically, we quantify each token's contribution to the dual weight matrix and measure the redundancy among token-induced updates. We then select the tokens that maximize information magnitude while minimizing this redundancy. To this end, we introduce Progressive Chunked Maximal Marginal Relevance (PC-MMR), which performs token selection at the chunk level and progressively expands the number of selected tokens. 
% By integrating our novel dual-form based metric with PC-MMR, our approach efficiently selects a subset of tokens that best preserves the information of the original dual weight matrix.
By integrating our novel dual-form based metric with PC-MMR, our approach efficiently selects a subset of tokens that retains the essential information of the original dual weight matrix.

Experiments on LVLMs across image and video benchmarks show that our pruning strategy consistently outperforms prior methods in both accuracy and efficiency.
These results suggest that the dual-form perspective provides a faithful lens for analyzing token importance in LVLMs and yields a better accuracy–efficiency trade-off.
% \clearpage

\section{Related Work and Problem Definition}
\subsection{Related Work}
\label{sec:rel works}
LVLMs \cite{liu2023visual, li2025llavaonevision, wang2024qwen2, yang2025qwen25} process extensive sequences of visual tokens, resulting in substantial computational overhead due to the quadratic complexity of attention. 
To mitigate this, prior works try to reduce the length of the visual tokens.
Methods such as FastV \cite{chen2024fastv}, PyramidDrop \cite{xing2024pyramiddrop}, SparseVLM \cite{zhang2025sparsevlm}, and PACT \cite{dhouib2025pact} select image tokens based on attention scores or query-token interactions of LLM. 
Meanwhile, approaches like VisionZip \cite{yang2025visionzip} and VisPruner \cite{zhang2025vispruner} rely on vision encoder attention statistics to measure token importance.

Beyond relying solely on attention score, several works \cite{dhouib2025pact, alvar2025divprune, wen2025dart} focus on mitigating redundancy among visual tokens. For instance, DivPrune \cite{alvar2025divprune} enforces diversity through minimum-distance constraints applied to the cosine similarity of hidden states.
More recently, hybrid approaches jointly optimize both token importance and redundancy. 
CDPruner \cite{zhang2025cdpruner} introduces conditional diversity maximization to balance redundancy with instruction relevance. 
SCOPE \cite{deng2025scope} jointly considers saliency and coverage to better preserve the semantic completeness of the selected token subset.
Similarly, PACT \cite{dhouib2025pact} leverages attention scores for token importance, while utilizing key cosine similarity to guide the clustering.
PyramidDrop \cite{xing2024pyramiddrop} progressively drops visual tokens with a pre-defined ratio in LVLM, and V2Drop \cite{v2drop} extends this approach by removing tokens that exhibit minimal variation across layers.
V2Drop \cite{v2drop} progressively removes visual tokens that exhibit minimal variation across LLM layers, while ZOO-Prune \cite{zooprune} applies zeroth-order perturbations to a lightweight projection layer to identify and selectively prune the most sensitive tokens in a training-free manner.

% Despite the impressive empirical success of visual token pruning in LVLMs, existing attention- or diversity-based approaches do not explicitly account for the underlying attention mechanism. 
% To overcome these limitations, we reformulate self-attention in its dual form. 
% This formulation allows us to analyze the structural contribution of each token, thereby introducing a novel, principled criterion that quantifies both importance and redundancy. 
% While prior methods have utilized the combination of attention scores and value norms for KV-cache pruning \cite{VATP}, our work formally derives this metric through the lens of dual-form attention. Building upon this dual-form perspective, we further introduce an redundancy metric and the PC-MMR mechanism.
Although visual token pruning in LVLMs has achieved impressive empirical success, existing attention- or diversity-based approaches largely rely on heuristics without explicitly accounting for the structural mechanics of attention. To address this limitation, we reformulate self-attention in its dual form. This theoretical lens allows us to formally derive token importance metrics that were previously adopted only empirically, such as the combination of attention scores and value norms used in KV-cache pruning \cite{VATP}. Building upon this unified foundation, we extend our analysis to quantify token-level information duplication, thereby introducing a novel redundancy metric and the PC-MMR mechanism for principled token selection.

\subsection{Problem Definition}
Large Vision-Language Models (LVLMs) process an image $I$ and a text $T$ to generate a response\footnote{For simplicity, we omit system prompt.}. Formally, the visual features from the encoder are mapped into the Large Language Model (LLM)'s hidden space via projection layer, yielding a sequence of $N_\text{img}$ visual tokens $\mathbf{Z}_\text{img} \in \mathbb{R}^{N_\text{img} \times d}$, where $d$ is the hidden dimension. Simultaneously, a text embedding layer maps the input text into $N_\text{text}$ tokens $\mathbf{Z}_\text{text} \in \mathbb{R}^{N_\text{text} \times d}$. Then, LLM processes the concatenated multimodal input $\mathbf{Z} =[\mathbf{Z}_\text{img} ; \mathbf{Z}_\text{text}] \in \mathbb{R}^{N \times d}$, where the total sequence length is $N = N_\text{img} + N_\text{text}$.
Given a target reduction ratio $\rho \in (0,1)$, a token reduction operator is applied at a specific LLM layer $l$ or projection layer, producing a compressed subset of $(1-\rho)N_\text{img}$ visual tokens.
The newly formed visual sequence is then concatenated with the $N_\text{text}$ text tokens and passed to subsequent transformer layers.

\section{Transformer Attention in Dual Form}
\label{sec:dual form}
Conventionally, standard attention \cite{vaswani2017attention} is viewed as an associative retrieval process, where similarity scores between a query and a set of keys determine the weights used to aggregate values. However, by exploiting the associativity of matrix multiplication, recent studies \cite{irie2022dual, dai2023can, ren2024towards} demonstrate that linear attention can be equivalently reinterpreted as a linear transformation.
For clarity, we first derive the dual form of unnormalized linear attention, where the implicit weight matrix construction from key–value pairs becomes most transparent.
% We then extend this analysis to standard attention to obtain its corresponding dual form. 
We then extend this analysis to standard attention, deriving its corresponding dual form.
For simplicity, we consider a single attention head throughout this section, as the derivation applies independently to each head. Throughout, let $\mathbf{q} \in \mathbb{R}^{1 \times d}$, $\mathbf{K} \in \mathbb{R}^{N \times d}$, and $\mathbf{V} \in \mathbb{R}^{N \times d_v}$ denote the query, keys, and values, respectively, all of which are obtained by projecting the input hidden states $\mathbf{Z}$, where $N$ denotes the number of tokens.

\subsection{Linear Attention in Dual Form}
Unnormalized linear attention is defined as
\begin{equation}
\text{LinearAttn}(\mathbf{q}, \mathbf{K}, \mathbf{V})
= \mathbf{q}\mathbf{K}^\top \mathbf{V}
= \sum_{i=1}^{N} \alpha_i \mathbf{v}_i,
\end{equation}
where $\alpha_i = \mathbf{q}\mathbf{k}_i^\top$. 
% In the primal form, attention is viewed as a weighted aggregation of values, acting as an associative retrieval mechanism.
% In this primal form, attention is a weighted sum of values.
By the associativity of matrix multiplication, this is equivalent to
\begin{equation}
\text{LinearAttn}(\mathbf{q}, \mathbf{K}, \mathbf{V})
= \mathbf{q}\mathbf{W}_N, \ %\quad
\mathbf{W}_N = \sum_{i=1}^{N} \mathbf{k}_i^\top \mathbf{v}_i,
\end{equation}
% \begin{equation}
% \text{LinearAttn}(\mathbf{q}, \mathbf{K}, \mathbf{V})
% = \mathbf{q}\mathbf{W}_N, \quad
% \mathbf{W}_N = \sum_{i=1}^{N} \mathbf{k}_i^\top \mathbf{v}_i,
% \end{equation}
where $\mathbf{W}_N \in \mathbb{R}^{d \times d_v}$, and $\mathbf{k}_i, \mathbf{v}_i$ are the $i$-th rows of $\mathbf{K}, \mathbf{V}$.
In its dual form, linear attention can be reinterpreted as a linear transformation of the query by the \textit{dual weight matrix} $\mathbf{W}_N$, constructed from the summation of outer products between key-value pairs.

As noted in prior work \cite{dai2023can, ren2024towards}, the attention update $\mathbf{k}_i^\top \mathbf{v}_i$ shares the algebraic form of a gradient-based linear layer update. In standard gradient descent with a loss $\mathcal{L}$, the weight matrix update takes the form $\mathbf{x}_i^\top \mathbf{e}_i$, where $\mathbf{x}_i \in \mathbb{R}^{1 \times d}$ is the input representation and $\mathbf{e}_i = -\nabla_{\mathbf{y}_i} \mathcal{L} \in \mathbb{R}^{1 \times d_v}$ is the learning signal where $\mathbf{y}_i \in \mathbb{R}^{1 \times d_v}$ is the output corresponding to the $\mathbf{x}_i$. This results in an accumulated weight update expressed as $\mathbf{W}_N = \sum_N \mathbf{x}_i^\top \mathbf{e}_i$.
Under this correspondence, each key $\mathbf{k}_i$ functions as an input representation of train sample, and each value $\mathbf{v}_i$ serves as the corresponding learning signal. 
Consequently, linear attention can be interpreted as an implicit optimization process in which each token contributes an update $\mathbf{k}_i^\top \mathbf{v}_i$ to the dual weight matrix $\mathbf{W}_N$, satisfying
\begin{equation}
\mathrm{rank}(\mathbf{k}_i^\top \mathbf{v}_i) = 1, 
\quad \forall i \in \{1,\dots,N\}.
\end{equation}
Therefore, each token contributes a rank-1 update to the dual weight matrix.\footnote{For simplicity, we refer to the dual weight matrix $\mathbf{W}_N$ as the dual weight in the remainder of this paper.}
% Consequently, linear attention can be interpreted as an implicit optimization process in which each token contributes a rank-1 update $\mathbf{k}_i^\top \mathbf{v}_i$ (where $\text{rank}(\mathbf{k}_i^\top \mathbf{v}_i) = 1$) to the dual weight matrix $\mathbf{W}_N$. 

\subsection{Softmax Attention in Dual Form}
Standard softmax attention is defined as
\begin{align}
\text{Attn}(\mathbf{q}, \mathbf{K}, \mathbf{V})
&= \text{softmax}\!\left(\tfrac{\mathbf{q}\mathbf{K}^\top}{\sqrt{d}}\right)\mathbf{V} \nonumber \\
&= \frac{\sum_{i=1}^{N} \exp(\mathbf{q}\mathbf{k}_i^\top/\sqrt{d})\,\mathbf{v}_i}{\sum_{j=1}^{N} \exp(\mathbf{q}\mathbf{k}_j^\top/\sqrt{d})}.
\end{align}
Let
$\kappa(\mathbf{x},\mathbf{y}) = \exp\!\bigl(\mathbf{x}\mathbf{y}^\top/\sqrt{d}\bigr)$
for $\mathbf{x}, \mathbf{y} \in \mathbb{R}^{1 \times d}$ denote the exponential kernel.
Since $\kappa$ is positive definite, there exists a Reproducing Kernel Hilbert Space (RKHS)
$\mathcal{H}$ and a feature map
\(
\phi : \mathbb{R}^{1 \times d} \to \mathcal{H}
\)
such that
\(
\kappa(\mathbf{x},\mathbf{y})
=
\langle \phi(\mathbf{x}), \phi(\mathbf{y}) \rangle_{\mathcal{H}}.
\)
Substituting the kernel expansion yields
\begin{equation}
\text{Attn}(\mathbf{q}, \mathbf{K}, \mathbf{V})
= \eta_N(\mathbf{q}) \sum_{i=1}^{N} \langle \phi(\mathbf{q}), \phi(\mathbf{k}_i) \rangle_{\mathcal{H}}\, \mathbf{v}_i,
\end{equation}
where $\eta_N(\mathbf{q}) = \bigl(\sum_{j=1}^{N}\kappa(\mathbf{q},\mathbf{k}_j)\bigr)^{-1}$.
Although the RKHS $\mathcal{H}$ is generally infinite-dimensional, we assume $\phi(\mathbf{x}) \in \mathbb{R}^{1 \times m}$ for some finite $m$ to derive the dual form (see \Cref{app:hs_infinite_features} for the formal statement of RKHS). Under this assumption, the inner product reduces to $\phi(\mathbf{x})\phi(\mathbf{y})^\top$. Factoring out $\phi(\mathbf{q})$ gives the dual form:
\begin{equation}
\label{eq:softmax dualform}
\text{Attn}(\mathbf{q}, \mathbf{K}, \mathbf{V})
= \phi(\mathbf{q})\,\eta_N(\mathbf{q})\,\mathbf{W}_N,
\end{equation}
where $\mathbf{W}_N = \sum_{i=1}^{N} \phi(\mathbf{k}_i)^\top \mathbf{v}_i$ is the accumulated dual weight built from rank-1 outer products. Each token induces a rank-1 update
\begin{equation}
\label{eq:softmax rank1 update}
\Delta \mathbf{W}_i = \phi(\mathbf{k}_i)^\top \mathbf{v}_i, \quad
\mathbf{W}_N = \sum_{i=1}^{N} \Delta \mathbf{W}_i.
\end{equation}
The key difference from linear attention is that $\mathbf{W}_N$ is constructed using the kernel-induced feature representations, mapping inputs from the Hilbert feature space $\mathcal{H}$ back to the original value space $\mathbb{R}^{d_v}$.
Notably, the \textit{effective dual weight} in softmax attention is given by the normalized form $\eta_N(\mathbf{q})\mathbf{W}_N$, reflecting the query-dependent scaling. 
\begin{figure*}[tb]
  \centering
  % \resizebox{\textwidth}{!}{\includegraphics[]{figs/test5.pdf}
  \resizebox{\textwidth}{!}{\includegraphics[]{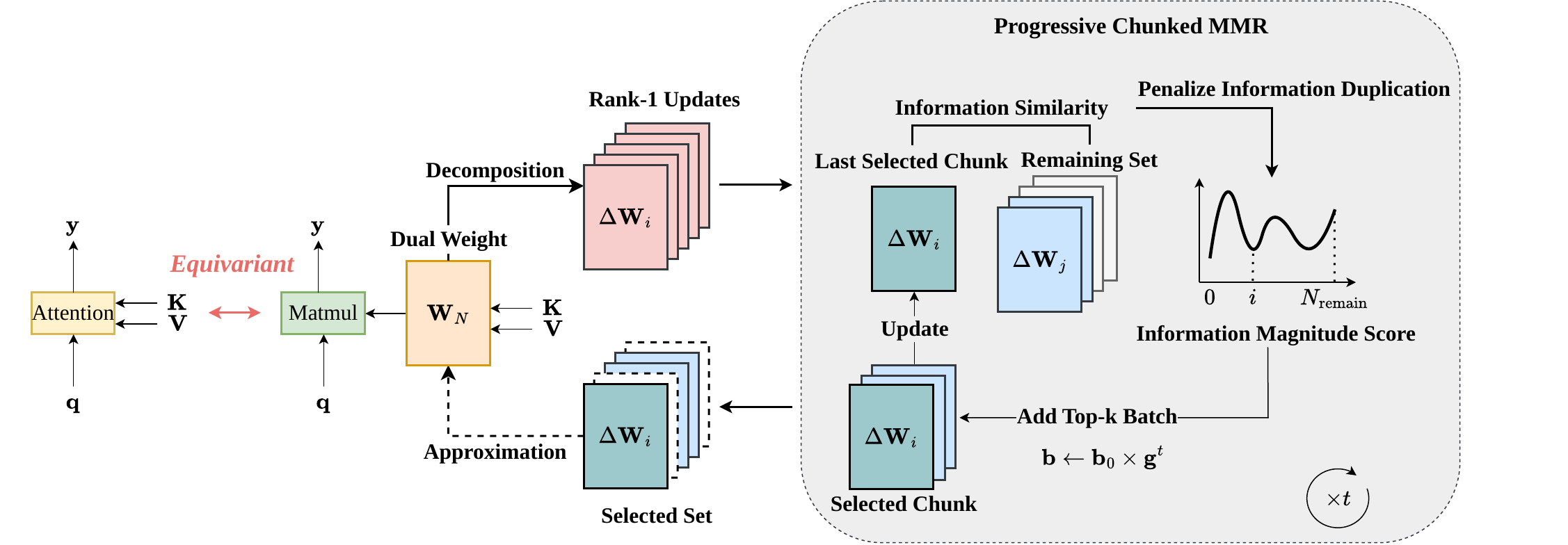}
  }
  \caption{\textbf{Overview of the Dual-Form Token Pruning framework.} Softmax attention is reinterpreted as a dual form via kernel mapping, where tokens generate rank-1 updates $\Delta \mathbf{W}_i = \phi(\mathbf{k}_i)^\top \mathbf{v}_i$. The Progressive Chunked MMR loop filters tokens based on information magnitude and duplication to efficiently approximate the dual weights. For simplicity, the normalization term $\eta_N(\mathbf{q})$ is omitted.
  }
  \label{fig:overview}
  \vspace{-2mm}
\end{figure*}

\section{Token Pruning in the Dual Form Perspective}
\label{sec:method}
Building upon this dual perspective, we define token informativeness through the lens of rank-1 weight updates. Specifically, the contribution of an individual token $i$ is encapsulated by its corresponding rank-1 update, $\Delta \mathbf{W}_i = \phi(\mathbf{k}_i)^\top \mathbf{v}_i$. We can therefore assess an individual token's informativeness by evaluating the information gain from its update to the dual weight, using two distinct criteria: \textbf{information magnitude} and \textbf{information duplication}. 
First, information magnitude is determined by the magnitude of the rank-1 update. A negligible magnitude (e.g., $\|\Delta\mathbf{W}_i\| \approx 0$) indicates that the token lacks salient information and has minimal influence on the dual weight; consequently, such tokens can be removed with minimal impact on performance. 
Second, information duplication is determined by the similarity between the rank-1 updates. 
% If an update is highly similar to those of other updates, the corresponding token is considered redundant.
If an update is highly similar to others, its corresponding token is considered redundant; consequently, pruning this duplicated information results in negligible loss.
% In the subsequent sections, we detail the mathematical formulations and the efficient token pruning algorithm considering two informativeness criteria.
In the subsequent sections, we detail \textbf{the mathematical formulations} and \textbf{the efficient token pruning algorithm} that consider two informativeness criteria.
An overview of the proposed token pruning framework is illustrated in \Cref{fig:overview}.

\subsection{Information magnitude: Dual-Form Importance Metric}
\label{sec:importance}
To address the first condition, we evaluate the individual information magnitude of each image token. Recall that the dual-form representation of attention consists of two components: (i) the accumulated dual weight $\mathbf{W}_N$, and (ii) the query-dependent normalization term $\eta_N(\mathbf{q})$.

% \subsubsection{Accumulated Kernel Weight Matrix} 
\noindent\textbf{Dual Weight Contribution.} 
According to the dual formulation, each token contributes a rank-1 update to the accumulated dual weight: $\Delta \mathbf{W}_i = \phi(\mathbf{k}_i)^\top\mathbf{v}_i$. We quantify the magnitude of the update using its Frobenius norm, $\|\Delta \mathbf{W}_i\|_F$. Using the identity $\|\mathbf{a}\mathbf{b}^\top\|_F = \|\mathbf{a}\|_2\|\mathbf{b}\|_2$, we obtain
\begin{equation}
\label{eq:rank-1 magnitude}
\|\Delta \mathbf{W}_i\|_F = \|\phi(\mathbf{k}_i)\|_2 \|\mathbf{v}_i\|_2,
\end{equation}
which characterizes the magnitude of the rank-1 update from token $i$. 
Consequently, tokens with a smaller $\|\Delta \mathbf{W}_i\|_F$ can be removed while minimizing the information loss.

\noindent\textbf{Normalization Contribution.}
Unlike unnormalized linear attention, standard softmax attention incorporates a query-dependent normalization term, $\eta_N(\mathbf{q}) = 1 / \sum_{j=1}^{N} \kappa(\mathbf{q}, \mathbf{k}_j)$.
Thus, an importance metric for token removal should account for both the token's contribution to the accumulated dual weight and its contribution to the normalization denominator. 
In particular, even if a token has a small dual-weight update magnitude $\|\Delta \mathbf{W}_i\|_F$, it may still have a non-negligible effect on the normalized output when it contributes substantially to the denominator.
The contribution of token $i$ to this denominator is given by the kernel similarity $\kappa(\mathbf{q}, \mathbf{k}_i)$:
\begin{equation}
\label{eq:query key sim}
\kappa(\mathbf{q}, \mathbf{k}_i) = \|\phi(\mathbf{q})\|_2\|\phi(\mathbf{k}_i)\|_2 \cos \theta_i
\end{equation}
where $\theta_i$ is the angle between $\phi(\mathbf{q})$ and $\phi(\mathbf{k}_i)$ in the Hilbert space $\mathcal{H}$. 

\noindent\textbf{Dual-Form Importance Metric.}
% To jointly account for these two aspects, the rank-1 update magnitude and the kernel-based alignment with the query, we introduce the following isolated informativeness metric:
To jointly account for the token's contribution to the accumulated dual weight and its contribution to the normalization term, we define the following isolated informativeness score:
\begin{equation}
\text{Score}_i = \kappa(\mathbf{q}, \mathbf{k}_i)\|\mathbf{v}_i\|_2
\end{equation}
This score corresponds to the query-conditioned contribution of token $i$ through the accumulated dual weight
$\left\|
\phi(\mathbf{q}) \Delta \mathbf{W}_i
\right\|_2,
$
where the same factor $\kappa(\mathbf{q}, \mathbf{k}_i)$ also determines the contribution of token $i$ to the normalization denominator.
Therefore, $\text{Score}_i$ captures both the token's dual-weight contribution and its query-dependent normalization contribution.
% Equivalently, using \cref{eq:rank-1 magnitude,eq:query key sim},
% $
% \text{Score}_i
% =
% \|\phi(\mathbf{q})\|_2
% \cos\theta_i
% \|\Delta \mathbf{W}_i\|_F .
% $
The specific query $\mathbf{q}$ utilized to compute this score remains a flexible design choice, adaptable to either text or image. To minimize the overhead, we aggregate the queries into a single representative vector, which reduces the cost to $\mathcal{O}(n)$ complexity. 
In practice, we use the mean of the text-token queries, defined as: $\mathbf{q}_T = \frac{1}{N_{\text{text}}} \sum_{j=1}^{N_{\text{text}}} \mathbf{q}_j.$
Alternative choices, including the mean image-token query $\mathbf{q}_I$ and the last text-token query $\mathbf{q}_{T_{[-1]}}$, are empirically studied in \Cref{sec:analysis}.

% To maintain computational efficiency, the specific query $\mathbf{q}$ utilized to compute this score is aggregated into a single representative vector rather than computing unique scores for every query in a sequence. For instance, in text-based tasks, we utilize the mean of the text-token queries: 
% \begin{equation}
% \mathbf{q}_T = \frac{1}{N_{\text{text}}} \sum_{j=1}^{N_{\text{text}}} \mathbf{q}_j
% \end{equation}
% This representative $\mathbf{q}_T$ ensures that preserved tokens are those most globally relevant to the current context's manifold, maintaining the structural integrity of the effective dual weight.

% \input{tables/4.1_pruning}

\begin{table*}[!t]
\vspace{-2mm}
\centering
\setlength{\tabcolsep}{3pt}
\resizebox{\textwidth}{!}{ 
\begin{tabular}{l cccccccccc |c}
\toprule
\textbf{Metric} & \textbf{AI2D} & \textbf{DocVQA} & \textbf{InfoVQA} & \textbf{MMBench} & \textbf{MME} & \textbf{MMMU} & \textbf{SciQA} & \textbf{TextVQA} & \textbf{MMStar} & \textbf{POPE} & \textbf{Avg. (\%)} \\
\midrule
\textbf{Baseline} & 81.3 & 87.1 & 66.0 & 80.7 & 1992.0 & 49.2 & 95.9 & 75.8 & 61.9 & 88.3 & 100.0 \\
\midrule\midrule

% ============================================================
\multicolumn{12}{l}{\textit{(a) Information magnitude metric analysis}} \\
\midrule
$\|\mathbf{v}_i\|_2$ & 79.6 & 79.1 & 51.2 & 79.5 & 1933.0 & 49.8 & 94.2 & 67.6 & 58.7 & 88.7 & 94.6 \\
$\|\mathbf{k}_i\|_2$ & 74.0 & 48.9 & 43.0 & 72.2 & 1558.8 & 46.6 & 83.5 & 61.7 & 49.5 & 78.1 & 81.1 \\
$\|\Delta \mathbf{W}_i\|_F$ & 79.0 & 69.8 & 43.8 & 79.4 & 1983.6 & 47.8 & 92.5 & 68.4 & 56.7 & 85.7 & 91.4 \\
\midrule
$\kappa(\mathbf{q_I}, \mathbf{k}_i)$ & 80.0 & 84.5 & 60.9 & 79.2 & 1998.5 & 49.3 & 92.0 & 75.3 & 57.5 & 87.0 & 97.3 \\
$\kappa(\mathbf{q_T}, \mathbf{k}_i)$ & 79.8 & 84.6 & 60.9 & 79.4 & 1973.4 & 49.1 & 91.8 & 75.5 & 58.2 & 87.1 & 97.3 \\
$\kappa(\mathbf{q_{T_{[-1]}}}, \mathbf{k}_i)$ & 78.8 & 85.0 & 59.8 & 79.0 & 1985.3 & 49.2 & 91.6 & 75.1 & 56.1 & 86.6 & 96.6 \\
\midrule
\rowcolor{gray!10}
$\kappa(\mathbf{q_I}, \mathbf{k}_i)\|\mathbf{v}_i\|_2$ & 80.2 & 85.5 & 60.9 & 79.5 & 2003.5 & 49.1 & 92.6 & 75.5 & 58.1 & 87.4 & 97.7 \\
\rowcolor{gray!10}
$\kappa(\mathbf{q_T}, \mathbf{k}_i)\|\mathbf{v}_i\|_2$ & 79.8 & 85.4 & 61.4 & 79.4 & 2001.5 & 49.0 & 92.0 & 75.6 & 59.1 & 87.6 & \textbf{97.8} \\

% ============================================================
\specialrule{1.2pt}{2pt}{2pt}
\multicolumn{12}{l}{\textit{(b) Information duplication metric analysis}} \\
\midrule
$\cos (\mathbf{v}_i,\mathbf{v}_j)$ & 80.3 & 77.8 & 50.1 & 78.6 & 1901.9 & 48.0 & 92.8 & 71.8 & 58.1 & 88.1 & 93.9 \\
$\cos (\mathbf{h}_i, \mathbf{h}_j)$ & 79.9 & 80.9 & 51.3 & 79.2 & 1900.4 & 48.6 & 92.6 & 72.9 & 58.5 & 88.6 & 94.9 \\
$\cos (\mathbf{k}_i,\mathbf{k}_j)$ & 80.0 & 81.5 & 55.1 & 80.1 & 1913.8 & 48.6 & 93.8 & 73.7 & 59.1 & 88.0 & 95.9 \\
$\cos(\phi(\mathbf{k}_i),\phi(\mathbf{k}_j))$ & 80.3 & 84.4 & 60.2 & 79.6 & 1965.9 & 49.0 & 92.7 & 75.0 & 59.9 & 88.5 & 97.6 \\
\rowcolor{gray!10}
$\cos (\Delta \mathbf{W}_i,\Delta \mathbf{W}_j)$ & 80.5 & 85.3 & 60.6 & 79.6 & 1985.0 & 48.9 & 92.9 & 75.2 & 60.0 & 88.1 & \textbf{97.9} \\
\bottomrule
\end{tabular}
}
\caption{\textbf{Analyses of information importance and duplication metrics on LLaVA-OneVision-7B.} 
\textbf{(a)} Effect of \emph{information magnitude} metrics. 
\textbf{(b)} Effect of \emph{information duplication} metrics for Progressive Chunked MMR. 
In all experiments, tokens are pruned after the fourth layer with a 35.3\% token budget. Scores are absolute performance, while Avg. (\%) indicates the relative average normalized by the full-token baseline.}
\label{tab:metric_analysis}
\end{table*}

% \begin{wraptable}{r}{0.6\textwidth}
% \vspace{-4.5mm}
% \centering
% \caption{\textbf{Comparison of dual-form importance metrics.}
% % The proposed dual-form importance metric can be factorized into three distinct components:
% % angular alignment $\cos\theta_i$, key magnitude $\|\phi(\mathbf{k}_i)\|_2$, and value magnitude $\|\mathbf{v}_i\|_2$.
% % Each metric is analyzed based on whether it accounts for these components ($\checkmark$) or not ($\times$).}
% }
% \label{tab:importance_metric_comparison}
% % \small
% \scriptsize
% \begin{tabular}{lcccc}
% \toprule
% \textbf{Metric} & \textbf{Formulation} 
% & $\cos\theta_i$ 
% & $\|\phi(\mathbf{k}_i)\|_2$ 
% & $\|\mathbf{v}_i\|_2$ \\
% \midrule

% Kernel Score 
% & $\kappa(\mathbf{q}, \mathbf{k}_i)$ 
% & \checkmark 
% & \checkmark 
% & $\times$ \\

% Gradient Norm 
% & $\|\Delta{\mathbf{W}_i}\|_F$ 
% & $\times$ 
% & \checkmark 
% & \checkmark \\

% Key Feature Norm 
% & $\|\phi(\mathbf{k}_i)\|_2$ 
% & $\times$ 
% & \checkmark 
% & $\times$ \\

% Value Norm 
% & $\|\mathbf{v}_i\|_2$ 
% & $\times$ 
% & $\times$ 
% & \checkmark \\

% \midrule

% \rowcolor{gray!10}

% \textbf{Ours} 
% & $\kappa(\mathbf{q}, \mathbf{k}_i)\|\mathbf{v}_i\|_2$ 
% & \checkmark 
% & \checkmark 
% & \checkmark \\
% \bottomrule
% \end{tabular}
% \vspace{-5mm}
% \end{wraptable}

\begin{table}[!t]
% \vspace{-4.5mm}
\centering
\small
\scriptsize
\setlength{\tabcolsep}{3pt}
\resizebox{\columnwidth}{!}{ 
\begin{tabular}{lcccc}
\toprule
\textbf{Metric} & \textbf{Formulation} 
& $\cos\theta_i$ 
& $\|\phi(\mathbf{k}_i)\|_2$ 
& $\|\mathbf{v}_i\|_2$ \\
\midrule

Kernel Score 
& $\kappa(\mathbf{q}, \mathbf{k}_i)$ 
& \checkmark 
& \checkmark 
& $\times$ \\

Gradient Norm 
& $\|\Delta{\mathbf{W}_i}\|_F$ 
& $\times$ 
& \checkmark 
& \checkmark \\

Key Feature Norm 
& $\|\phi(\mathbf{k}_i)\|_2$ 
& $\times$ 
& \checkmark 
& $\times$ \\

Value Norm 
& $\|\mathbf{v}_i\|_2$ 
& $\times$ 
& $\times$ 
& \checkmark \\

\midrule

\rowcolor{gray!10}

\textbf{Ours} 
& $\kappa(\mathbf{q}, \mathbf{k}_i)\|\mathbf{v}_i\|_2$ 
& \checkmark 
& \checkmark 
& \checkmark \\
\bottomrule
\end{tabular}
}
\caption{\textbf{Comparison of dual-form impor. metrics.}
% The proposed dual-form importance metric can be factorized into three distinct components:
% angular alignment $\cos\theta_i$, key magnitude $\|\phi(\mathbf{k}_i)\|_2$, and value magnitude $\|\mathbf{v}_i\|_2$.
% Each metric is analyzed based on whether it accounts for these components ($\checkmark$) or not ($\times$).}
}
\label{tab:importance_metric_comparison}
\vspace{-5mm}
\end{table}

\subsection{Information Duplication: Dual-Form Similarity Metric}
\label{sec:similarity}
For the second criterion, we quantify redundancy between tokens through the similarity of their induced rank-1 updates.
A high similarity indicates that the updates from both tokens contain duplicated information, whereas a near-zero similarity implies that the updated information is orthogonal and introduces novel information to the effective dual weight. 
To compute this similarity, we derive the inner product between the corresponding update matrices.
Utilizing the reproducing property $\langle \phi(\mathbf{k}_i), \phi(\mathbf{k}_j) \rangle_{\mathcal{H}} = \kappa(\mathbf{k}_i, \mathbf{k}_j)$, the Frobenius inner product between two rank-1 updates is given by:
\begin{align}
\langle \Delta \mathbf{W}_i, \Delta \mathbf{W}_j \rangle_F
=
(\mathbf{v}_i \cdot \mathbf{v}_j)\,
\kappa(\mathbf{k}_i, \mathbf{k}_j).
\end{align}
Normalizing by 
$\|\Delta \mathbf{W}_i\|_F
=
\sqrt{\kappa(\mathbf{k}_i, \mathbf{k}_i)}\,\|\mathbf{v}_i\|_2$,
we define the dual weight similarity as
% \begin{align}
% \text{S}_{ij}
% =
% \frac{\langle \Delta \mathbf{W}_i, \Delta \mathbf{W}_j \rangle_F}
% {\|\Delta \mathbf{W}_i\|_F \|\Delta \mathbf{W}_j\|_F}
% =
% \left(
% \frac{\mathbf{v}_i \cdot \mathbf{v}_j}
% {\|\mathbf{v}_i\|_2 \|\mathbf{v}_j\|_2}
% \right)
% \left(
% \frac{\kappa(\mathbf{k}_i, \mathbf{k}_j)}
% {\sqrt{\kappa(\mathbf{k}_i, \mathbf{k}_i)\kappa(\mathbf{k}_j, \mathbf{k}_j)}}
% \right).
% \end{align}
\begin{equation}
\begin{split}
\text{S}_{ij}
&= \frac{\langle \Delta \mathbf{W}_i, \Delta \mathbf{W}_j \rangle_F}
        {\|\Delta \mathbf{W}_i\|_F \|\Delta \mathbf{W}_j\|_F} \\
&= \left(\frac{\mathbf{v}_i \cdot \mathbf{v}_j}{\|\mathbf{v}_i\|_2 \|\mathbf{v}_j\|_2}\right)
   \left(\frac{\kappa(\mathbf{k}_i, \mathbf{k}_j)}{\sqrt{\kappa(\mathbf{k}_i, \mathbf{k}_i)\kappa(\mathbf{k}_j, \mathbf{k}_j)}}\right).
\end{split}
\end{equation}
Thus, $S_{ij}$ is factorized into the cosine similarity between value vectors and the cosine similarity between keys in the RKHS. We use the squared similarity $S_{ij}^2$ as the metric to prioritize updates whose similarity is close to zero.
For the exponential kernel $\kappa$, the latter term reduces to a Gaussian RBF kernel,
\(
\exp\!\left(
-\frac{\|\mathbf{k}_i - \mathbf{k}_j\|_2^2}
{2\sqrt{d_h}}
\right)
\).
\subsection{Dual Form-based Token Pruning}
\label{sec:dual-form based token pruning}
Finally, we combine dual-form based metrics via Maximal Marginal Relevance (MMR)~\cite{carbonell1998mmr}, which penalizes candidates similar to already selected tokens.
% Finally, we combine information magnitude and duplication into a unified token pruning framework. Magnitude-only selection often produces redundant updates, causing rank collapse in the accumulated dual weight, while duplication-only selection may retain tokens with negligible magnitudes. 
% Thus, we aim to select tokens with large information magnitude and low duplication. To this end, we adopt Maximal Marginal Relevance (MMR) \cite{carbonell1998mmr}, which penalizes candidates similar to already selected tokens.
Given the full token set $\mathcal{U}$ and the selected set $\mathcal{C}$, MMR chooses the next token as
$
i^* = \arg\max_{i \in \mathcal{U} \setminus \mathcal{C}} [ \lambda \cdot P_i - (1 - \lambda) \max_{j \in \mathcal{C}} S_{i,j} ],
$
where $P_i$, $S_{i,j}$, and $\lambda$ denote the dual-form importance metric ($\text{Score}_i$), dual-form similarity metric, and similarity penalty ratio, respectively.

However, standard MMR is greedy-selection based, requiring $(1-\rho)N_{\text{img}}$ iterations and $\mathcal{O}(N_{\text{img}}^2)$ pairwise comparisons, which creates a bottleneck during inference. To overcome this, we propose Progressive Chunked MMR, which accounts for both criteria with lower overhead.

\noindent\textbf{Progressive Chunked MMR} uses chunk-based selection and comparison. The goal is minimizing the information-duplicated tokens, while minimizing the computational and time cost of the MMR. It first computes each token's information magnitude in $O(n)$ time, then selects tokens in chunks rather than one by one. At each iteration, it selects a chunk of top-magnitude tokens ($L_{\text{new}}$) and computes the duplication matrix $S$ between the selected set $\mathcal{C}$ and the remaining tokens $\mathcal{U}$. The maximum redundancy $\mathbf{s}_{\max}$ is then used as a multiplicative penalty for the remaining candidates, favoring tokens with high magnitude and low duplication. The similarities are kept during the iteration, reducing the computational cost. To further balance the trade-off, the chunk size is progressively increased ($b \leftarrow b \cdot g$), accelerating selection with limited performance degradation. The detailed procedure is provided in \Cref{alg:parallel_mmr} in the Appendix.
% \clearpage
\section{Analysis}
\label{sec:analysis}

\begin{table*}[!t]
\vspace{-2mm}
\centering
\resizebox{\linewidth}{!}{
\begin{tabular}{l cccccccccc|c}
\toprule
\textbf{Method} & \textbf{AI2D} & \textbf{DocVQA} & \textbf{InfoVQA} & \textbf{MMB.} & \textbf{MME} & \textbf{MMMU} & \textbf{SciQA} & \textbf{TextVQA} & \textbf{MMStar} & \textbf{POPE} & \textbf{Avg. (\%)} \\
\midrule
\rowcolor{gray!10}
\multicolumn{12}{c}{\textit{Upper Bound, All Tokens} ($\mathbf{100\%}$)} \\
Baseline & 81.3 & 87.1 & 66.0 & 80.7 & 1992.0 & 49.2 & 95.9 & 75.8 & 61.9 & 88.3 & 100.0 \\
\rowcolor{gray!10}
\multicolumn{12}{c}{\textit{Token Budget: 35.3\%} \textcolor{Green}{($\downarrow\mathbf{64.7\%}$)}} \\
FastV & 79.6 & 84.5 & 58.0 & 79.1 & 1984.5 & 48.4 & 91.8 & 75.6 & 58.7 & 86.5 & 96.7 \\
PACT & 79.9 & 84.2 & 60.9 & 79.3 & 1947.8 & 48.3 & 92.7 & 75.4 & 59.1 & 88.3 & 97.3 \\
VisionZip & 76.2 & 47.6 & 36.0 & 77.7 & 1916.2 & 46.4 & 90.1 & 63.0 & 53.8 & 87.8 & 85.3 \\
DivPrune & 79.2 & 71.0 & 46.7 & 78.4 & 1850.3 & 49.6 & 91.8 & 67.7 & 57.3 & 88.3 & 91.8 \\
CDPruner & 78.0 & 60.3 & 50.6 & 79.9 & 1914.3 & 46.9 & 92.8 & 71.2 & 57.9 & 88.5 & 91.7 \\
SCOPE & 78.1 & 48.0 & 44.2 & 78.8 & 1923.9 & 46.6 & 89.8 & 69.1 & 55.1 & 88.3 & 88.1 \\
\rowcolor{gray!40}
\textbf{Ours} & 80.5 & 85.3 & 60.6 & 79.6 & 1985.0 & 48.9 & 92.9 & 75.2 & 60.0 & 88.1 &  \textbf{97.9} \\
\rowcolor{gray!15}
\multicolumn{12}{c}{\textit{Token Budget: 22.2\%} \textcolor{Green}{($\downarrow\mathbf{77.8\%}$)}} \\
FastV & 78.6 & 80.6 & 52.1 & 78.5 & 1994.5 & 47.7 & 91.1 & 74.6 & 57.0 & 84.7 & 94.4 \\
PACT & 79.7 & 80.5 & 54.9 & 77.8 & 1905.4 & 47.8 & 90.9 & 73.8 & 57.3 & 87.3 & 94.7 \\
VisionZip & 73.0 & 37.2 & 32.6 & 78.8 & 1840.6 & 46.9 & 88.5 & 54.5 & 53.3 & 85.9 & 81.4 \\
DivPrune & 77.0 & 60.4 & 39.2 & 77.1 & 1758.0 & 48.7 & 89.3 & 61.2 & 55.5 & 87.3 & 86.9 \\
CDPruner & 77.8 & 56.9 & 52.4 & 79.4 & 1889.0 & 47.0 & 91.5 & 68.9 & 57.9 & 88.3 & 91.0 \\
SCOPE & 75.9 & 33.3 & 37.3 & 77.7 & 1872.1 & 46.4 & 87.3 & 64.0 & 51.3 & 87.7 & 83.0 \\
\rowcolor{gray!40}
\textbf{Ours} & 79.3 & 82.2 & 56.5 & 78.8 & 1977.9 & 47.8 & 91.5 & 74.4 & 57.3 & 87.3 &  \textbf{95.6} \\
\rowcolor{gray!10}
\multicolumn{12}{c}{\textit{Token Budget: 11.1\%} \textcolor{Green}{($\downarrow\mathbf{88.9\%}$)}} \\
FastV & 77.0 & 67.4 & 41.4 & 77.0 & 1955.8 & 47.9 & 89.3 & 71.3 & 53.2 & 79.2 & 88.8 \\
PACT & 75.8 & 74.6 & 49.3 & 76.5 & 1883.7 & 45.1 & 89.4 & 70.2 & 55.4 & 83.3 & 90.4 \\
VisionZip & 70.8 & 26.9 & 29.3 & 74.7 & 1788.2 & 46.0 & 84.8 & 43.5 & 47.3 & 83.1 & 75.4 \\
DivPrune & 74.4 & 45.2 & 33.3 & 74.5 & 1638.0 & 47.3 & 85.0 & 52.1 & 48.5 & 86.1 & 79.8 \\
CDPruner & 77.8 & 56.2 & 54.7 & 78.2 & 1849.5 & 47.4 & 89.1 & 66.7 & 56.7 & 88.4 & 90.2 \\
SCOPE & 71.7 & 21.6 & 30.7 & 74.0 & 1743.6 & 45.8 & 86.5 & 55.7 & 49.1 & 85.7 & 77.2 \\
\rowcolor{gray!40}
\textbf{Ours} & 76.8 & 71.9 & 45.5 & 76.5 & 1947.0 & 47.4 & 89.5 & 70.7 & 54.6 & 84.0 &  \textbf{90.5} \\
\bottomrule
\end{tabular}
}
\caption{\textbf{Performance comparison of different pruning methods on LLaVA-OneVision-7B.}
Scores are absolute performance, while Avg. (\%) indicates the relative average normalized by the full-token baseline.}
\label{tab:llava-onevision-pruning}
\vspace{-5mm}
\end{table*}
% While our method uses both information magnitude and duplication, we first analyze these two criteria in detail. For analysis, we conduct our experiments under fixed token budgets using LLaVA-OneVision-7B. Further analyses, including formal justification, comparison between information magnitude and duplication, experiments and experimental details can be found in \Cref{supp:further_analyses,supp:imp_details} and \Cref{sec:experiments}.

% The benefit of the information duplication metric grows as the pruning budget tightens (\cref{tab:ablation_budget_analysis}). At larger token budgets, the importance metric alone identifies informative tokens reliably. At smaller budgets, many tokens share comparable importance scores, and accounting for duplication yields additional gains.
% \cref{supp:imp_details}
While our method leverages both information magnitude and duplication, we first analyze these two criteria in detail. For analysis, we conduct experiments under fixed token budgets using LLaVA-OneVision-7B. 
The experimental setup is described in \Cref{sec:experiments}, and ablation studies on the two criteria and the effect of RoPE are provided in \Cref{supp:further_analyses}.

\subsection{Information Magnitude}
% In \Cref{sec:importance}, we introduced two key factors to consider when selecting tokens for pruning: 
In \Cref{sec:importance}, we introduced two key factors when selecting tokens: 
(i) the dual weight $\|\Delta \mathbf{W}_i\|_F$ and (ii) the normalization term $\eta(\mathbf{q})$. 
The combined information magnitude  metric decomposes into angular alignment, key magnitude, and value magnitude—dimensions along which existing metrics are categorized in \Cref{tab:importance_metric_comparison}.

Most existing methods adopt kernel-based scoring strategies. In particular, methods based on $\kappa(\mathbf{q}, \mathbf{k}_i)$ \cite{chen2024fastv, dhouib2025pact} primarily capture angular alignment and key magnitude while neglecting value-side magnitude. Magnitude-based metrics such as $\|\Delta \mathbf{W}_i\|_F$ emphasize key and value magnitudes yet overlook the angular alignment component. 
In contrast, the proposed metric $\kappa(\mathbf{q}, \mathbf{k}_i)\|\mathbf{v}_i\|_2$ jointly accounts for normalization and dual-weight perturbation, aligning with the dual-form structure.

As shown in \cref{tab:metric_analysis}, our proposed metric $\kappa(\mathbf{q}_T, \mathbf{k}_i)\|\mathbf{v}_i\|_2$, which jointly considers alignment, key magnitude, and value magnitude, achieves the best overall performance. When evaluated individually, alignment-based methods of the form $\kappa(\mathbf{q}, \mathbf{k}_i)$ attain comparable performance at 97.3\% despite omitting value magnitude. 
Conversely, relying solely on magnitude leads to lower performance: the normalization term distorts the contribution of high-magnitude tokens, rendering the resulting scores an unreliable indicator of importance. 
% We further expect this effect to intensify as the number of vision tokens grows, since the normalization term aggregates the alignments between the query and all keys. Consistent with this, our experiments show that magnitude-only metrics suffer a substantial performance drop on high-resolution datasets such as DocVQA and InfoVQA. Among magnitude-only metrics, using value magnitude consistently outperforms using key magnitude. 
% This degradation intensifies on high-resolution datasets (e.g., DocVQA, InfoVQA), where the number of vision tokens grows. Among magnitude-only metrics, value magnitude outperforms key magnitude.
This degradation intensifies on high-resolution datasets (e.g., DocVQA, InfoVQA), where more vision tokens amplify the normalization effect. Among magnitude-only metrics, value magnitude outperforms key magnitude.

Finally, relying on the last text token $\mathbf{q}_{T_{[-1]}}$ degrades performance relative to using the aggregated queries $\mathbf{q}_I$ or $\mathbf{q}_T$. By integrating these three critical components with aggregated queries, our proposed metric achieves optimal performance.

% \subsection{Information Duplication}
% As detailed in \cref{sec:similarity}, we formulate information duplication between two tokens as the similarity between their rank-1 dual weight updates, which decomposes into key similarity within the RKHS and value similarity.
% While this aligns with diversity-based token pruning methods \cite{dhouib2025pact, alvar2025divprune, deng2025scope, zhang2025cdpruner}, it introduces a crucial distinction: we define redundancy through the overlap of induced dual updates rather than proximity in the raw feature space.
% To investigate which representation space best captures information duplication, we evaluate similarities across five variants: value, hidden-state, key, kernelized-key, and dual-weight spaces.

% \input{tables/4.2_mmr}

% \Cref{tab:mmr analysis} reveals a clear trend: key-space similarity consistently outperforms value-space similarity, with hidden-state similarity in between. Applying cosine similarity in the RKHS further improves performance by 1.7 percentage points (pp), highlighting the importance of kernel-induced similarity. Consequently, dual weight similarity achieves the best performance at 97.9\%, demonstrating that the similarity between rank-1 dual weight updates effectively captures information redundancy.

\begin{table*}[!t]
\centering
\resizebox{\linewidth}{!}{
\begin{tabular}{l cccccccccc |c}
\toprule
\textbf{Method} & \textbf{AI2D} & \textbf{DocVQA} & \textbf{InfoVQA} & \textbf{MMB.} & \textbf{MME} & \textbf{MMMU} & \textbf{SciQA} & \textbf{TextVQA} & \textbf{MMStar} & \textbf{POPE} & \textbf{Avg.(\%)} \\
\midrule
\rowcolor{gray!10}
\multicolumn{12}{c}{\textit{Upper Bound, All Tokens} ($\mathbf{100\%}$)} \\
Baseline
& 82.4 & 94.6 & 80.4 & 84.2 & 2310.2 & 50.7 & 88.3 & 82.8 & 62.3 & 87.6 & 100.0 \\
\rowcolor{gray!10}
\multicolumn{12}{c}{\textit{Token Budget: 35.3\%} \textcolor{Green}{($\downarrow\mathbf{64.7\%}$)}} \\
FastV
& 79.5 & 80.6 & 64.3 & 81.0 & 2006.6 & 49.8 & 84.0 & 80.5 & 56.4 & 82.1 & 92.0 \\
PACT
& 79.6 & 69.7 & 65.8 & 78.3 & 2244.4 & 47.6 & 85.1 & 79.5 & 55.9 & 81.2 & 91.1 \\
VisionZip
& 81.1 & 89.4 & 63.7 & 82.6 & 2307.6 & 51.0 & 87.8 & 79.5 & 59.3 & 85.2 & 96.0 \\
DivPrune
& 80.9 & 86.0 & 60.7 & 81.5 & 2274.0 & 48.7 & 86.6 & 78.9 & 57.8 & 84.9 & 94.0 \\
SparseVLM
& 78.4 & 68.1 & 55.1 & 81.2 & 2102.0 & 49.1 & 83.7 & 78.5 & 56.3 & 86.2 & 89.8 \\
PyramidDrop
& 74.9 & 37.0 & 34.1 & 71.9 & 2193.0 & 44.8 & 84.5 & 75.4 & 55.2 & 83.3 & 81.2 \\
V2Drop
& 78.5 & 54.3 & 47.3 & 80.2 & 2206.0 & 47.3 & 86.1 & 58.4 & 57.6 & 85.5 & 85.4 \\
% ZooPrune
% & 58.7 & 44.1 & 36.7 & 74.1 & 1512.0 & 42.1 & 81.7 & 24.5 & 49.0 & 82.1 & 69.4 \\
SCOPE
& 76.8 & 90.5 & 68.1 & 78.6 & 2339.3 & 49.6 & 87.6 & 81.1 & 59.8 & 85.8 & 95.7 \\

CDPruner
& 80.3 & 76.1 & 53.3 & 81.5 & 2231.5 & 49.6 & 85.9 & 73.4 & 57.6 & 81.7 & 90.9 \\

\rowcolor{gray!40}
\textbf{Ours} 
& 81.8 & 90.5 & 66.8 & 82.3 & 2319.8 & 50.0 & 87.2 & 80.7 & 61.0 & 85.7 &  \textbf{96.8 }\\
\rowcolor{gray!10}
\multicolumn{12}{c}{\textit{Token Budget: 22.2\%} \textcolor{Green}{($\downarrow\mathbf{77.8\%}$)}} \\
FastV 
& 76.9 & 81.4 & 53.4 & 79.0 & 2236.2 & 49.0 & 83.0 & 77.0 & 53.0 & 77.9 & 89.4 \\
PACT
& 78.5 & 71.6 & 57.2 & 77.2 & 2222.5 & 47.4 & 83.9 & 76.8 & 55.0 & 82.8 & 89.4 \\
VisionZip 
& 78.7 & 78.4 & 50.9 & 80.6 & 2229.0 & 49.1 & 85.5 & 74.5 & 55.0 & 83.5 & 90.2 \\
DivPrune
& 78.0 & 76.3 & 48.0 & 79.5 & 2193.4 & 49.1 & 84.9 & 75.1 & 54.0 & 83.8 & 89.2 \\
SparseVLM
& 76.2 & 63.4 & 52.0 & 79.8 & 2053.0 & 47.7 & 82.3 & 76.6 & 53.4 & 85.8 & 87.1 \\
PyramidDrop
& 73.1 & 29.8 & 28.1 & 70.4 & 2112.0 & 44.2 & 84.0 & 70.5 & 52.4 & 81.1 & 77.4 \\
V2Drop
& 78.1 & 39.9 & 40.2 & 79.0 & 2220.0 & 47.0 & 85.9 & 46.2 & 56.1 & 84.4 & 80.9 \\
% ZooPrune
% & 55.8 & 24.0 & 29.6 & 71.1 & 1377.0 & 39.9 & 80.4 & 17.2 & 43.7 & 76.0 & 62.1 \\
SCOPE
& 80.3 & 84.1 & 55.9 & 80.4 & 2256.9 & 48.7 & 86.0 & 78.6 & 56.8 & 85.0 & 92.6 \\
CDPruner
& 78.9 & 66.2 & 44.4 & 79.1 & 2108.6 & 49.7 & 83.5 & 67.6 & 53.2 & 79.4 & 85.8 \\

\rowcolor{gray!40}
\textbf{Ours} 
& 80.5 & 84.1 & 54.9 & 81.6 & 2242.7 & 49.6 & 86.0 & 78.0 & 58.6 & 84.1 &  \textbf{93.0} \\
\rowcolor{gray!10}
\multicolumn{12}{c}{\textit{Token Budget: 11.1\%} \textcolor{Green}{($\downarrow\mathbf{88.9\%}$)}} \\
FastV
& 71.6 & 63.7 & 39.5 & 76.7 & 2006.6 & 48.4 & 82.8 & 69.2 & 47.5 & 68.5 & 80.9 \\
PACT
& 72.9 & 51.9 & 39.3 & 76.2 & 2108.0 & 46.4 & 82.9 & 67.2 & 48.6 & 78.8 & 80.8 \\
VisionZip
& 73.8 & 56.5 & 36.9 & 75.8 & 1995.1 & 46.1 & 83.1 & 63.3 & 49.4 & 78.9 & 80.4 \\
SparseVLM
& 71.1 & 53.8 & 45.4 & 76.5 & 1913.0 & 46.2 & 80.2 & 71.0 & 50.1 & 82.9 & 81.6 \\
PyramidDrop
& 71.8 & 22.0 & 21.5 & 69.5 & 1957.0 & 43.8 & 82.0 & 59.9 & 50.2 & 78.3 & 72.6 \\
V2Drop
& 73.7 & 25.2 & 28.2 & 52.1 & 2148.0 & 42.0 & 83.8 & 35.4 & 51.3 & 81.8 & 70.2 \\
% ZooPrune
% & 57.4 & 13.6 & 25.4 & 62.5 & 1179.0 & 37.9 & 76.3 & 13.8 & 39.4 & 66.2 & 55.8 \\
SCOPE
& 74.6 & 64.0 & 39.7 & 79.0 & 2100.8 & 48.2 & 83.9 & 72.2 & 51.8 & 81.5 & 84.6 \\
DivPrune
& 72.4 & 57.5 & 34.8 & 77.1 & 2033.7 & 46.0 & 81.5 & 64.6 & 48.5 & 79.2 & 80.2 \\
CDPruner
& 74.1 & 49.5 & 33.4 & 75.7 & 1925.0 & 47.3 & 82.7 & 55.9 & 46.8 & 73.5 & 77.2 \\

\rowcolor{gray!40}
\textbf{Ours} 
& 76.4 & 64.0 & 37.8 & 77.7 & 2152.3 & 48.8 & 84.0 & 71.6 & 53.3 & 78.2 &  \textbf{84.7} \\
\bottomrule
\end{tabular}
}
\caption{\textbf{Performance comparison of different pruning methods on Qwen2.5-VL-7B.}
Scores indicate absolute performance, while Avg. (\%) represents the relative average normalized by the full-token baseline.
}
\label{tab:qwen2.5-vl}
\end{table*}

\subsection{Information Duplication}
As detailed in \Cref{sec:similarity}, we formulate information duplication between two tokens as the similarity between their rank-1 dual weight updates—a metric that naturally decomposes into key similarity within the RKHS and value similarity. While this perspective partially aligns with diversity-based token pruning methods \cite{dhouib2025pact, alvar2025divprune, deng2025scope, zhang2025cdpruner} that seek to minimize redundancy among retained tokens, it introduces a crucial distinction: we define redundancy in terms of the overlap of induced dual updates rather than proximity in the raw feature space. Conventionally, such similarity is measured using two widely adopted metrics: cosine similarity between keys and cosine similarity between hidden states. To investigate which representation space most effectively captures information duplication, we evaluate similarities across five variants: value, hidden-state, key, kernelized-key, and dual-weight spaces.

\Cref{tab:metric_analysis} reveals a clear performance trend across representation spaces. Similarity measured in the key-space consistently outperforms value-space, while hidden-state similarity yields intermediate performance. Notably, applying cosine similarity in the RKHS improves performance by 1.7 percentage points (pp), underscoring the importance of kernel-induced similarity. Dual weight similarity achieves the best performance, surpassing both conventional metrics. These results demonstrate that the similarity between rank-1 dual weight updates effectively captures information redundancy.

Notably, the benefit of the duplication metric grows as the pruning budget tightens (\Cref{tab:ablation_budget_analysis}): under tighter budgets, many tokens exhibit comparable importance scores, so accounting for duplication yields additional gains. 

% As shown in Table 3, dual weight similarity achieves the best performance (97.9\%), outperforming key-space (95.9\%), hidden-state (94.9\%), and value-space (93.9\%) similarities. Notably, kernel-induced similarity (RKHS) improves over raw key similarity by 1.7 pp, confirming that the dual-form structure better captures information redundancy.
\section{Experiments}

\label{sec:experiments}
\subsection{Experimental Setup}
\noindent\textbf{Evaluation Benchmarks.}
We evaluate on a diverse set of image and video vision–language benchmarks, covering text-centric understanding, knowledge-intensive reasoning, general perception, and hallucination for images, as well as temporal reasoning for videos. 
We report each benchmark's official metric and the relative performance $\text{Score}_{\text{Pruned}}/\text{Score}_{\text{Baseline}}\times 100$ with respect to the vanilla model. Full benchmark list and descriptions are provided in \Cref{supp:benchmarks}.

\noindent\textbf{Implementation Details.}
We use LLaVA-OneVision-7B~\cite{li2025llavaonevision} and Qwen2.5-VL-7B~\cite{yang2025qwen25} with FlashAttention2~\cite{dao2024flashattention}, and evaluate with lmms-eval~\cite{zhang2025lmms}. All baselines are reproduced under our setup for fair comparison. For our method, we set the pruning layer $l=4$, initial chunk size $b_0=2$, growth factor $g=2$, and duplication penalty $\lambda=5$. Full implementation details and hyperparameter ablations are deferred to \Cref{supp:imp_details,supp:further_analyses}.
\subsection{Image Understanding}

\begin{table*}[!t]
  \centering
  
  \begin{minipage}[t]{0.48\textwidth}
    \centering
    \resizebox{\linewidth}{!}{
      \begin{tabular}{l ccccc}
      \toprule
      \textbf{Method} 
      & \textbf{Ego.} 
      & \textbf{V.MME} 
      & \textbf{MLVU} 
      & \textbf{$\text{NExT}$} 
      & \textbf{Avg.(\%)} \\
      
      \rowcolor{gray!10}
      \multicolumn{6}{c}{\textit{Upper Bound, All Tokens} ($\mathbf{100\%}$)} \\
      
      Baseline & 62.4 & 58.4 & 64.7 & 79.3 & 100.0\% \\
      
      \rowcolor{gray!10}
      \multicolumn{6}{c}{\textit{Token Budget: 11.1\%} \textcolor{Green}{($\downarrow\mathbf{88.9\%}$)}} \\
      FastV & 58.2 & 53.4 & 57.3 & 74.5 & 92.0\% \\
      PACT & 61.0 & 54.5 & 61.1 & 76.7 & 95.7\% \\
      DivPrune & 58.8 & 53.8 & 61.0 & 76.8 & 93.9\% \\
      
      \rowcolor{gray!40}
      \textbf{Ours} & 62.2 & 54.8 & 60.9 & 77.6 & 96.9\% \\
      \bottomrule
      \end{tabular}
    }
    \caption{\textbf{Performance comparison on various Video Benchmarks on LLaVA-OneVision-7B.}}
    \label{tab:video}
  \end{minipage}%
  \hfill
  \begin{minipage}[t]{0.48\textwidth}
    \centering
    
    \resizebox{\linewidth}{!}{
      \setlength{\extrarowheight}{0pt}
      \renewcommand{\arraystretch}{0.87}
      \begin{tabular}{l cccc}
      \toprule
      \textbf{Method} & \textbf{Latency (s)} & \textbf{Prefill (s)}  & \textbf{VRAM(GB)} & \textbf{Acc.(\%)} \\
      
      \rowcolor{gray!10}
      \multicolumn{5}{c}{\textit{Upper Bound, All Tokens} ($\mathbf{100\%}$)} \\
      Baseline & 1.37 & 1.13 & 17.8 & 95.9 \\
      \rowcolor{gray!10}
      \multicolumn{5}{c}{\textit{Token Budget: 11.1\%} \textcolor{Green}{($\downarrow\mathbf{88.9\%}$)}} \\
      FastV & 0.88 & 0.61& 16.8 & 89.3 \\
      PACT & 0.78 & 0.43 & 16.1 & 89.4 \\
      DivPrune & 1.09 & 0.89 & 16.3& 85 \\
      CDPruner & 1.08 & 0.88 & 16.3& 89.1 \\
      SCOPE & 2.29 & 2.09 & 16.3& 86.5 \\
      \rowcolor{gray!40}
      \textbf{Ours} & 0.64 & 0.39 & 16.1 & 89.5 \\
      \bottomrule
      \end{tabular}
    }
    \caption{\textbf{Efficiency comparison of different pruning methods on LLaVA-OneVision-7B.}}
    \label{tab:efficiency}
  \end{minipage}
\vspace{-6mm}
\end{table*}
\noindent\textbf{LLaVA-OneVision.}
\Cref{tab:llava-onevision-pruning} shows the performance across various visual benchmarks using LLaVA-OneVision-7B as the backbone. 
The experimental results demonstrate that our method achieves state-of-the-art (SOTA) performance across all three token budgets.
% Among the competing methods, PACT and FastV show the second- and third- best performance, while other approaches exhibit much lower scores.
Our method consistently outperforms other methods, which rely on attention score based importance measurement and diversity criteria, across all evaluated token budgets. In particular, at a token budget of 22.2\%, it surpasses PACT and FastV by 0.9 percentage points (pp) and 1.2 pp, respectively.
These results suggest that the importance and similarity metrics derived from our dual-form perspective provide a more precise estimation of token informativeness compared to previous empirical based methods. Since some previous methods are incompatible with AnyRes in LLaVA-OneVision, we conduct the comparison using a different model.

\noindent\textbf{Qwen2.5-VL.}
We further evaluate on Qwen2.5-VL-7B, which differs from LLaVA-OneVision-7B in its vision encoder, LLM backbone, and visual token pipeline (dynamic resolution vs.\ AnyRes).
As shown in \Cref{tab:qwen2.5-vl}, our method achieves the best performance across all token budgets, outperforming VisionZip by 2.6 pp on average.
% The baseline ranking also shifts: encoder-level VisionZip becomes more competitive, while LLM-level methods (FastV, PACT) drop. Several baselines degrade noticeably on DocVQA.
The baseline ranking also shifts: VisionZip becomes more competitive, while the performance of LLM-level methods (FastV, PACT) drops. Several baselines degrade noticeably on DocVQA.
% These results indicate that prior pruning methods are sensitive to architecture and task, while our dual-form pruning generalizes more robustly: by measuring actual information updates in self-attention rather than architecture-specific patterns, it adapts to different visual token structures such as dynamic resolution.
These results indicate that prior pruning methods are sensitive to architecture and task, while our dual-form pruning generalizes more robustly. By measuring actual information updates in self-attention rather than architecture-specific patterns, it adapts to different visual token structures such as dynamic resolution.

\subsection{Video Understanding}
% To evaluate our method on video benchmarks, we conduct experiments using LLaVA-OneVision across four distinct datasets. We compare our method with FastV \cite{chen2024fastv}, which leverages LLM attention scores for token selection, PACT \cite{dhouib2025pact}, which leverages LLM attention scores and hidden states similarity for clustering, and DivPrune \cite{alvar2025divprune}, which promotes diversity among visual tokens.

In \Cref{tab:video}, we conduct experiments on various video benchmarks. Our proposed approach demonstrates a better trade-off between performance and computational efficiency across benchmarks. On long-video reasoning tasks such as EgoSchema, our method effectively preserves the essential visual semantics required for understanding long sequences, exceeding the second-ranked method by 1.2 pp. Furthermore, on NextQA and Video-MME, which require precise causal reasoning and diverse contextual comprehension, our approach achieves better performance compared to other methods, while removing a significant number of tokens.

\subsection{Computational Efficiency}
\label{efficiency}
To evaluate the practical efficiency of our proposed method, we comprehensively profile the latency, prefill time, and maximum GPU memory allocation (VRAM). All measurements are conducted on a single A6000 GPU using the ScienceQA dataset.
For a fair comparison, we standardize the attention implementation across all baselines to PyTorch's \cite{paszke2019pytorchimperativestylehighperformance} Scaled Dot-Product Attention (SDPA), which leverages FlashAttention (FA) where applicable.
% As shown in \Cref{tab:efficiency}, although DivPrune performs pruning at the projector stage, its reliance on a greedy algorithm incurs computational and memory overhead that ultimately increases generation time. 
% FastV and PACT, in turn, depend on full attention for scoring and clustering, respectively, and therefore achieve only marginal speedups while yielding slightly lower performance than our method. In contrast, our approach performs token reduction within the LLM stage and is fully compatible with FA. Notably, while competitive methods such as CDPruner and SCOPE, which account diversity in token pruning, are slowed down by their greedy procedures, the proposed Progressive Chunked MMR adds only negligible overhead (0.02s), demonstrating the efficiency of our method.
As shown in \Cref{tab:efficiency}, greedy-based methods (DivPrune, CDPruner, SCOPE) suffer from selection overhead, while attention-dependent methods (FastV, PACT) cannot leverage FA for scoring and clustering, respectively. In contrast, our method is fully FA-compatible and Progressive Chunked MMR adds only negligible overhead (0.02s).

% By identifying and keeping only the most informative visual tokens based on our dual-form perspective, our method minimizes performance drops. As a result, our approach achieves the best trade-off by reducing both generation time and peak VRAM, demonstrating its practical value for deploying LVLMs.

% \input{sections/6_related_works}

\section{Conclusion}
\label{sec:conclusion}

% In this paper, we address the computation and memory bottlenecks in LVLMs via visual token pruning. Departing from existing empirical token pruning methods, we introduce a theoretically grounded token pruning framework based on the dual-form perspective of self-attention. By reinterpreting softmax attention as an implicit linear transformation updated by key-value pairs, we formulate token pruning as the selection of the informative rank-1 updates. This perspective enables us to derive a novel, training-free token pruning that explicitly balances information magnitude and directional diversity to identify the most crucial visual tokens.
% Extensive evaluations across diverse benchmarks demonstrate that our principled approach consistently outperforms prior pruning strategies. Notably, our method validates that the dual-form perspective offers a faithful and interpretable lens for analyzing token importance, translating a theoretical framework into practical and scalable efficiency gains for LVLMs.

In this paper, we address the computational bottlenecks in LVLMs through visual token pruning. Unlike existing empirical methods, we introduce another perspective based on the dual form of self-attention. By reinterpreting softmax attention as an implicit linear transformation updated by key-value pairs, we formulate token pruning as selecting the most informative rank-1 updates. This perspective allows us to derive a novel, training-free token pruning method that explicitly balances information magnitude and duplication to identify the most informative visual tokens. Extensive experiments across diverse benchmarks show that our approach consistently outperforms prior pruning methods, in terms of both peformance and efficiency. Notably, our results demonstrate that the dual-form perspective provides a reliable and interpretable way to analyze token importance.
Furthermore, we provide an in-depth analysis of two key aspects related to the informativeness of each token under the dual-form perspective.

\section*{Limitations}
While our dual-form framework offers a unified view of token pruning, it has several limitations.
First, our analysis is layer-wise. We address the deletion error of a single attention layer's dual weight, but does not characterize how this error propagates through subsequent layers, nor how it interacts with residual connections, FFN nonlinearities, and downstream attention. Empirical robustness across pruning layers suggests the local criterion is informative in practice, but a multi-layer analysis remains open.
Second, the theoretical decomposition is derived for a fixed query. We aggregate text-token queries into a single representative to keep selection at $\mathcal{O}(n)$, which empirically performs best among the alternatives we tested. However, this aggregation is not derived from the dual-form perspective itself, and the choice of q remains a design decision rather than a consequence of the theory.
Third, Progressive Chunked MMR is a practical approximation to greedy MMR rather than a method with provable selection guarantees. The chunk schedule and the multiplicative penalty form are tuned empirically.
We leave these directions to future work.

\section*{Acknowledgements}
This work was supported by Institute of Information \& Communications Technology Planning \& Evaluation(IITP) grant funded by the Korea government(MSIT) (RS-2024-00439020, Developing Sustainable, Real-Time Generative AI for Multimodal Interaction, SW Starlab).
% \bibliography{anthology,refs}
\bibliography{refs}

\appendix
\clearpage
% \begin{center}
% \Large\textbf{Token Pruning as Implicit Weight Pruning in Large Vision Language Models: Supplementary Materials}
% % \maketitle
% \end{center}
\appendix
% \onecolumn
\section*{Appendix}

% \crefalias{section}{appendix}

% \renewcommand{\thefigure}{\Alph{section}.\arabic{figure}}
% \renewcommand{\thetable}{\Alph{section}.\arabic{table}}

\counterwithin{figure}{section}
\counterwithin{table}{section}

\renewcommand{\thesection}{\Alph{section}}
\renewcommand{\thefigure}{\thesection.\arabic{figure}}
\renewcommand{\thetable}{\thesection.\arabic{table}}

\renewcommand{\theHsection}{\thesection}
\renewcommand{\theHfigure}{\theHsection.\arabic{figure}}
\renewcommand{\theHtable}{\theHsection.\arabic{table}}

\section{Theoretical Derivation}
Here, we provide theoretical derivation and justification of two criteria of IWP.
\subsection{Hilbert--Schmidt Interpretation of Infinite-Dimensional Feature Maps}
\label{app:hs_infinite_features}

The analysis in the main text is presented using the finite-feature
notation
\begin{align}
\mathbf{W}_N
=
\sum_{i=1}^{N}\phi(\mathbf{k}_i)^{\!\top}\mathbf{v}_i
\in \mathbb{R}^{m\times d_v},
\end{align}
where \(\phi(\mathbf{x})\) is treated as a row vector in
\(\mathbb{R}^{1\times m}\).  This notation is convenient when the kernel
admits a finite-dimensional feature representation.  However, for the
exponential dot-product kernel
\begin{align}
\kappa(\mathbf{x},\mathbf{y})
=
\exp(\mathbf{x}\mathbf{y}^{\!\top}/\sqrt{d_h}),
\end{align}
the associated feature expansion contains monomials of all degrees, and
the corresponding RKHS is infinite-dimensional. This appendix formalizes the finite-feature notation
through Hilbert--Schmidt operators. Note that regardless of finite- or infinite-dimensional RKHSs, the same identities remain valid because
they depend only on finitely many kernel evaluations.
\paragraph{Setup.}
Let \(\mathcal{H}\) be the RKHS associated with \(\kappa\), and let
\(\phi:\mathbb{R}^{d_h}\to\mathcal{H}\) denote its canonical feature map.
Then
\begin{align}
\langle \phi(\mathbf{x}),\phi(\mathbf{y})\rangle_{\mathcal{H}}
=
\kappa(\mathbf{x},\mathbf{y}), \nonumber\\
\|\phi(\mathbf{x})\|_{\mathcal{H}}^{2}
=
\kappa(\mathbf{x},\mathbf{x}).
\label{eq:reproducing}
\end{align}

For \(\mathbf{k}\in\mathbb{R}^{d_h}\) and
\(\mathbf{v}\in\mathbb{R}^{d_v}\), define the rank-one operator
\begin{align}
\phi(\mathbf{k})\otimes\mathbf{v}:\mathcal{H}\to\mathbb{R}^{d_v},
\nonumber\\
h\mapsto
\langle h,\phi(\mathbf{k})\rangle_{\mathcal{H}}\mathbf{v}.
\end{align}
This operator is Hilbert--Schmidt.  We write
\(\mathrm{HS}(\mathcal{H},\mathbb{R}^{d_v})\) for the space of
Hilbert--Schmidt operators from \(\mathcal{H}\) to \(\mathbb{R}^{d_v}\).
For rank-one operators, the Hilbert--Schmidt inner product is
\begin{align}
\langle f\otimes\mathbf{w},\,g\otimes\mathbf{u}\rangle_{\mathrm{HS}}
=
\langle f,g\rangle_{\mathcal{H}}
\langle \mathbf{w},\mathbf{u}\rangle_2 .
\end{align}
Accordingly, the dual weight can be interpreted as
\begin{equation}
\mathbf{W}_N
:=
\sum_{i=1}^{N}\phi(\mathbf{k}_i)\otimes\mathbf{v}_i
\in
\mathrm{HS}(\mathcal{H},\mathbb{R}^{d_v}).
\label{eq:WN_HS}
\end{equation}
Since \(\mathbf{W}_N\) is a finite sum of rank-one operators, it is
finite-rank and hence Hilbert--Schmidt, regardless of whether
\(\mathcal{H}\) is finite- or infinite-dimensional.

\paragraph{Rank-one Hilbert--Schmidt norm.}
\begin{lemma}[Rank-one Hilbert--Schmidt norm]
\label{lem:hs_norm}
For every token \(i\),
\begin{align}
\|\Delta\mathbf{W}_i\|_{\mathrm{HS}}
&=
\|\phi(\mathbf{k}_i)\|_{\mathcal{H}}\|\mathbf{v}_i\|_2 \nonumber\\
&=
\sqrt{\kappa(\mathbf{k}_i,\mathbf{k}_i)}\,\|\mathbf{v}_i\|_2 .
\end{align}
\end{lemma}

\begin{proof}
For any \(f\in\mathcal{H}\) and \(\mathbf{w}\in\mathbb{R}^{d_v}\),
\begin{align}
\|f\otimes\mathbf{w}\|_{\mathrm{HS}}^2
&=
\langle f,f\rangle_{\mathcal{H}}
\langle \mathbf{w},\mathbf{w}\rangle_2 \nonumber\\
&=
\|f\|_{\mathcal{H}}^2\|\mathbf{w}\|_2^2 .
\end{align}
Taking \(f=\phi(\mathbf{k}_i)\) and \(\mathbf{w}=\mathbf{v}_i\), and
using \eqref{eq:reproducing}, gives the result.
\end{proof}

\paragraph{Pairwise Hilbert--Schmidt inner products.}
\begin{lemma}[Pairwise Hilbert--Schmidt inner product]
\label{lem:hs_inner}
For any two tokens \(i\) and \(j\),
\begin{equation}
\langle
\Delta\mathbf{W}_i,
\Delta\mathbf{W}_j
\rangle_{\mathrm{HS}}
=
(\mathbf{v}_i\!\cdot\!\mathbf{v}_j)\,
\kappa(\mathbf{k}_i,\mathbf{k}_j).
\label{eq:hs_inner}
\end{equation}
\end{lemma}

\begin{proof}
By the rank-one Hilbert--Schmidt inner product identity,

\begin{align*}
\langle
&\phi(\mathbf{k}_i)\otimes\mathbf{v}_i,\,
\phi(\mathbf{k}_j)\otimes\mathbf{v}_j
\rangle_{\mathrm{HS}} \\
&=
\langle
\phi(\mathbf{k}_i),
\phi(\mathbf{k}_j)
\rangle_{\mathcal{H}}
\langle
\mathbf{v}_i,
\mathbf{v}_j
\rangle_2  \\
&=
\kappa(\mathbf{k}_i,\mathbf{k}_j)
(\mathbf{v}_i\!\cdot\!\mathbf{v}_j),
\end{align*}

where the last equality follows from \eqref{eq:reproducing}.
\end{proof}

\paragraph{Cosine factorization.}
\begin{proposition}[Cosine factorization]
\label{prop:sim_factor}
Assume \(\mathbf{v}_i,\mathbf{v}_j\neq 0\) and
\(\kappa(\mathbf{k}_i,\mathbf{k}_i),
\kappa(\mathbf{k}_j,\mathbf{k}_j)>0\).  Then the Hilbert--Schmidt cosine
between two rank-one updates factors as
\begin{align}
&\frac{
\langle
\Delta\mathbf{W}_i,
\Delta\mathbf{W}_j
\rangle_{\mathrm{HS}}
}{
\|\Delta\mathbf{W}_i\|_{\mathrm{HS}}
\|\Delta\mathbf{W}_j\|_{\mathrm{HS}}
}
= \nonumber\\
&
\Bigg(
\frac{\mathbf{v}_i\!\cdot\!\mathbf{v}_j}
{\|\mathbf{v}_i\|_2\|\mathbf{v}_j\|_2}
\Bigg)
\Bigg(
\frac{\kappa(\mathbf{k}_i,\mathbf{k}_j)}
{\sqrt{
\kappa(\mathbf{k}_i,\mathbf{k}_i)
\kappa(\mathbf{k}_j,\mathbf{k}_j)
}}
\Bigg).
\label{eq:sim_factor}
\end{align}
Thus the similarity between two dual-weight updates decomposes into a
value-space cosine and a normalized kernel similarity between their
keys.  In particular, for the exponential dot-product kernel,
\begin{equation}
\frac{\kappa(\mathbf{k}_i,\mathbf{k}_j)}
{\sqrt{
\kappa(\mathbf{k}_i,\mathbf{k}_i)
\kappa(\mathbf{k}_j,\mathbf{k}_j)
}}
=
\exp\!\left(
-\frac{
\|\mathbf{k}_i-\mathbf{k}_j\|_2^2
}{2\sqrt{d_h}}
\right).
\end{equation}
\end{proposition}

\begin{proof}
The factorization follows immediately from
\Cref{lem:hs_norm,lem:hs_inner}.  For the exponential dot-product
kernel,
\[
\begin{aligned}
\frac{\kappa(\mathbf{x},\mathbf{y})}
{\sqrt{\kappa(\mathbf{x},\mathbf{x})\kappa(\mathbf{y},\mathbf{y})}}
&=
\exp\!\left(
\frac{\mathbf{x}\mathbf{y}^{\!\top}}{\sqrt{d_h}}
-
\frac{\|\mathbf{x}\|_2^2+\|\mathbf{y}\|_2^2}{2\sqrt{d_h}}
\right) \\
&=
\exp\!\left(
-\frac{\|\mathbf{x}-\mathbf{y}\|_2^2}{2\sqrt{d_h}}
\right),
\end{aligned}
\]
which proves the stated expression.
\end{proof}

\paragraph{Score decomposition.}
\begin{proposition}[Angular--key--value decomposition of the score]
\label{prop:score_factor}
Assume \(\phi(\mathbf{q})\neq 0\) and \(\phi(\mathbf{k}_i)\neq 0\).
The score used in \Cref{sec:importance} satisfies
\begin{align}
\mathrm{Score}_i
&:=
\kappa(\mathbf{q},\mathbf{k}_i)\|\mathbf{v}_i\|_2 \nonumber\\
&=
{
\|\phi(\mathbf{q})\|_{\mathcal{H}}
}
\cdot
{
\cos\theta_i
}
\cdot
{
\|\Delta\mathbf{W}_i\|_{\mathrm{HS}}
}
\end{align}
where
\begin{equation}
\cos\theta_i
:=
\frac{
\langle
\phi(\mathbf{q}),
\phi(\mathbf{k}_i)
\rangle_{\mathcal{H}}
}{
\|\phi(\mathbf{q})\|_{\mathcal{H}}
\|\phi(\mathbf{k}_i)\|_{\mathcal{H}}
}.
\end{equation}
\end{proposition}

\begin{proof}
By the reproducing property,
\begin{align}
\kappa(\mathbf{q},\mathbf{k}_i)
&=
\langle
\phi(\mathbf{q}),
\phi(\mathbf{k}_i)
\rangle_{\mathcal{H}}\\
&=
\|\phi(\mathbf{q})\|_{\mathcal{H}}
\|\phi(\mathbf{k}_i)\|_{\mathcal{H}}
\cos\theta_i .
\end{align}
Multiplying both sides by \(\|\mathbf{v}_i\|_2\) and using
\Cref{lem:hs_norm},
\begin{equation}
\|\phi(\mathbf{k}_i)\|_{\mathcal{H}}\|\mathbf{v}_i\|_2
=
\|\Delta\mathbf{W}_i\|_{\mathrm{HS}},
\end{equation}
yields the result.
\end{proof}

\subsection{Exact Deletion Effect vs. Proposed Importance Score}
\label{app:exact_deletion_effect}
While the proposed information magnitude metric is exactly equivalent to the deletion effect in the unnormalized dual-weight space, it may differ in standard softmax attention due to the presence of the normalization term. In this section, we formally derive the exact token deletion effect, analyze its differences from our proposed importance score, and empirically validate the impact of the underlying approximations.

To formalize this, let $a_i = \kappa(\mathbf{q}, \mathbf{k}_i)$ denotes the query-key similarity for token $i$, and $Z = \sum_j a_j$ be the normalization factor. The normalized attention weight is given by $w_i = a_i / Z$, and the standard softmax attention output computed over the given tokens is $\mathbf{o} = \sum_{j} w_j \mathbf{v}_j$. 
The exact effect of deleting token $i$ on the normalized output is defined as:
\begin{equation}
\lVert \mathbf{o} - \mathbf{o}_{-i} \rVert_2 = \frac{w_i}{1 - w_i}\lVert \mathbf{v}_i - \mathbf{o} \rVert_2,
\end{equation}
where $\mathbf{o}_{-i} = \sum_{j \neq i} \frac{a_j}{Z - a_i} \mathbf{v}_j$ denote the attention output after removing token $i$.
Substituting $w_i = a_i / Z$ into the equation yields:
\begin{equation}
\lVert \mathbf{o} - \mathbf{o}_{-i} \rVert_2 = \frac{1}{Z} \frac{a_i}{1 - w_i}\lVert \mathbf{v}_i - \mathbf{o} \rVert_2.
\end{equation}
For a fixed query, the normalizer $Z$ is a constant shared across all tokens. Therefore, omitting the factor $1/Z$ does not alter the relative ordering of tokens. This provides a rank-equivalent exact deletion effect $E_i$:
\begin{equation}
E_i \propto \frac{a_i}{1 - w_i}\lVert \mathbf{v}_i - \mathbf{o} \rVert_2.
\end{equation}
Our proposed score, $\mathrm{Score}_i = a_i\lVert \mathbf{v}_i \rVert_2 = \kappa(\mathbf{q},\mathbf{k}_i)\lVert \mathbf{v}_i \rVert_2$, approximates this rank-equivalent metric $E_i$ by applying two approximations: (1) dropping the normalization factor $\frac{1}{1 - w_i}$ and (2) replacing the centered value vector norm $\lVert \mathbf{v}_i - \mathbf{o} \rVert_2$ with the uncentered norm $\lVert \mathbf{v}_i \rVert_2$. 

\paragraph{Dropping the Normalization Factor.}
Removing the factor $\frac{1}{1 - w_i}$ from $E_i$ yields an intermediate centered score: $a_i \lVert \mathbf{v}_i - \mathbf{o} \rVert_2 = \kappa(\mathbf{q},\mathbf{k}_i)\lVert \mathbf{v}_i - \mathbf{o} \rVert_2$. This simplification effectively ignores an $\frac{1}{1 - w_i}$ correction term.

Empirically, omitting this factor does not alter token rankings. In Qwen2.5-VL-7B, the attention shares of individual visual tokens remain sufficiently small: the 95th percentile of per-head $w_i$ is 0.18, with a worst-case of 0.35. Consequently, the omitted factor $\frac{1}{1-w_i}$ remains strictly below 1.6 and exhibits minimal variance across tokens, making it mathematically too small to reorder the importance scores. Rescoring every token with the exact coefficient ($\frac{a_i}{1-w_i}\lVert \mathbf{v}_i \rVert_2$ instead of $a_i\lVert \mathbf{v}_i \rVert_2$) changed no token ranking in any of the evaluated cases.

\paragraph{Uncentered Norm vs. Centered Norm.}
Our metric further replaces the centered norm $\lVert \mathbf{v}_i - \mathbf{o} \rVert_2$ with the uncentered norm $\lVert \mathbf{v}_i \rVert_2$. This query-independent replacement avoids the computational overhead of computing the full attention output $\mathbf{o}$ prior to the pruning process.

We evaluated the centered score variant, $\kappa(\mathbf{q},\mathbf{k}_i)\lVert \mathbf{v}_i - \mathbf{o} \rVert_2$, across ten image benchmarks at a 35.3\% budget. The empirical results demonstrate minimal performance fluctuations, ranging from $-0.65$ to $+0.43$ percentage points compared to the uncentered metric. Moreover, the top-$K$ overlap between the centered and uncentered scores is strictly $\ge 0.947$, and the rank correlation of the centered score with the exact deletion effect is 0.997. These findings confirm that the uncentered norm $\lVert \mathbf{v}_i \rVert_2$ is an accurate and computationally efficient surrogate, with the centered score serving as a validated variant if the overhead of computing $\mathbf{o}$ is acceptable.

\section{Benchmarks and Models}
\label{supp:benchmarks}
\subsection{Evaluation benchmarks}

\subsubsection{Text-centric visual understanding benchmarks.} \hfill

\noindent \textbf{AI2D}~\cite{kembhavi2016diagram} A dataset consisting of over 4,900 grade school science diagrams annotated with rich ground-truth information, including object classes, relationships, and text elements. It is designed to evaluate a model's ability to understand complex informational graphics and answer multiple-choice questions that require reasoning over both visual and textual diagrammatic components.

\noindent \textbf{TextVQA}~\cite{singh2019towards} A dataset designed to evaluate models on visual question answering tasks that require reading and reasoning about text present within images. It contains 28,408 images sourced from the Open Images dataset, paired with 45,336 questions. Models must effectively perform optical character recognition (OCR) and integrate the recognized text with visual context to provide accurate answers. We use the validation split for evaluation.

\noindent \textbf{DocVQA}~\cite{mathew2021docvqa} A large-scale visual question answering dataset focused on document understanding. It comprises 12,767 diverse document images, ranging from printed forms to typed letters, accompanied by over 50,000 questions. It is designed to test a model's ability to comprehend text, spatial layout, and structural information within scanned documents. We use the validation split for evaluation.

\noindent \textbf{Infographic VQA}~\cite{mathew2022infographicvqa} A benchmark extending the document understanding task to complex infographics. It requires models to process and reason over diverse textual, graphical, and layout elements, often necessitating numerical reasoning and data extraction from charts and visual graphics.

\subsubsection{Knowledge-intensive multimodal reasoning benchmarks.} \hfill

\noindent \textbf{MMBench}~\cite{liu2024mmbench} A comprehensive evaluation pipeline for assessing the fine-grained capabilities of Vision-Language Models. It features around 3,000 multiple-choice questions covering 20 ability dimensions. It introduces a circular evaluation strategy, where the choices of a question are circularly shifted to robustly test whether a model consistently predicts the correct answer regardless of the option's position.

\noindent \textbf{MMMU}~\cite{yue2024mmmu} A Massive Multi-discipline Multimodal Understanding benchmark designed to evaluate college-level reasoning abilities in multimodal models. It contains 11.5K meticulously collected questions spanning 6 core disciplines (Art \& Design, Business, Science, Health \& Medicine, Humanities \& Social Science, and Tech \& Engineering) and 30 subjects, requiring deep domain knowledge and complex reasoning. We use the validation split for evaluation.

\noindent \textbf{MMStar}~\cite{chen2024we} An evaluation benchmark featuring 1,500 highly challenging samples specifically curated to test true visual dependency in multimodal models. The dataset is filtered to remove samples that can be answered through text-only guessing or common sense without viewing the image, thus strictly evaluating the model's authentic visual understanding and multimodal reasoning capabilities.

\noindent \textbf{ScienceQA}~\cite{lu2022learn} A large-scale multimodal multiple-choice question answering benchmark focused on diverse scientific domains. It contains 21,208 questions spanning natural science, language science, and social science, categorized into 26 topics, 127 categories, and 379 skills. Among them, 48.7\% include image context, 48.2\% include text context, and 30.8\% include both. A majority of questions are annotated with grounded lectures and detailed explanations, offering external knowledge and reasoning to support the correct answer. We use the test split that includes image context for evaluation.

\subsubsection{Comprehensive perception, cognition, and hallucination benchmarks.} \hfill

\noindent \textbf{MME}~\cite{fu2025mme} A comprehensive evaluation benchmark for multimodal large language models that spans 14 diverse subtasks. These tasks are broadly categorized into perception (e.g., object recognition, OCR, color, position) and cognition (e.g., commonsense reasoning, numerical calculation, text translation). It uses instruction-based prompts to systematically measure both basic visual recognition and higher-level reasoning.

\noindent \textbf{POPE}~\cite{li2023evaluating} A Polling-based Object Probing Evaluation benchmark designed to systematically assess object hallucination in large vision-language models. Instead of relying on generative descriptions, POPE formulates the evaluation as a binary classification task (Yes/No questions) inquiring about the presence of specific objects in an image. It evaluates hallucination under different settings, including random, popular, and adversarial object probing.

\begin{algorithm*}[t]
\caption{Progressive Chunked Maximal Marginal Relevance}
\label{alg:parallel_mmr}
\begin{algorithmic}[1]
\REQUIRE Visual tokens $\mathbf{Z}_{img} \in \mathbb{R}^{N_{img} \times d}$; target size $M$; chunk size $b$; growth factor $g$; penalty strength $\lambda$.
\STATE $\mathcal{C} \leftarrow \emptyset$ \quad // Initialize selected index set
\STATE $\mathcal{U} \leftarrow \{1, \dots, N_{img}\}$ \quad // Initialize unselected candidates
\STATE \textit{// Calculate information magnitude scores}
\STATE $P \leftarrow \{P_1, \dots, P_{N_{img}}\}$, where $P_i=\kappa(q_{T}, k_i)\|v_i\|_2$
\WHILE{$|\mathcal{C}| < M$}
    \STATE $k \leftarrow \min(b, M - |\mathcal{C}|)$ \quad // Current selection budget
    \STATE $L_{new} \leftarrow \text{arg\,top-}k_{i \in \mathcal{U}} \{ P_i \}$
    \STATE $\mathcal{C} \leftarrow \mathcal{C} \cup L_{new}$
    \STATE $\mathcal{U} \leftarrow \mathcal{U} \setminus L_{new}$
    \STATE \textit{// Chunked Information Duplication computation}
    \STATE $S \leftarrow S(\mathcal{V}_{\mathcal{U}}, \mathcal{V}_{L_{new}})$
    \STATE $s_{max} \leftarrow \max(S, \mathrm{dim}{=}1)$
    \STATE $P_{\mathcal{U}} \leftarrow P_{\mathcal{U}} \odot \max(0.01, 1 - \lambda \cdot s_{max})$ \label{line:penalty_update}
    \STATE $b \leftarrow b \cdot g$
\ENDWHILE
\RETURN $\mathcal{C}$
\end{algorithmic}
\end{algorithm*}

\subsubsection{Video understanding and temporal reasoning benchmarks.}\hfill

\noindent \textbf{EgoSchema}~\cite{mangalam2023egoschema} A very long-form egocentric video understanding benchmark spanning over 250 hours of real-world video footage. It features multiple-choice questions that require complex temporal reasoning over 3-minute video clips. The benchmark specifically targets tasks that cannot be solved by looking at a single frame, pushing models to understand long-term human activities and intentions.

\noindent \textbf{Video-MME}~\cite{fu2025videomme} A comprehensive benchmark for evaluating multimodal large language models on video analysis. It covers various video lengths (short, medium, and long up to 1 hour) and spans diverse domains including movies, sports, and daily life. The evaluation tasks encompass spatial perception, temporal reasoning, and overall narrative comprehension.

\noindent \textbf{MLVU}~\cite{zhou2025mlvu} A Multi-task Long Video Understanding benchmark designed to assess the capabilities of models in processing and reasoning over extended video sequences. It incorporates a wide array of tasks that require capturing both fine-grained temporal details and global semantic context across long video streams.

\noindent \textbf{NExT-QA}~\cite{xiao2021nextqa} A video question answering benchmark focusing heavily on causal and temporal action reasoning. Moving beyond simple action recognition, NExT-QA challenges models to answer ``why'' and ``how'' questions, requiring an understanding of the causal relationships between events and the temporal order of actions within the video.

\subsection{Evaluated models}

To comprehensively evaluate the generality and effectiveness of our proposed token pruning method, we conduct experiments on two state-of-the-art open-source Vision-Language Models (VLMs) that employ different visual encoding strategies.

\noindent \textbf{LLaVA-OneVision ~\cite{li2025llavaonevision}} LLaVA-OneVision is a family of large multimodal models designed for robust visual understanding across single images, multiple images, and videos. We specifically utilize the model that integrates the Qwen2-7B language model backbone with a SigLIP visual encoder. To handle high-resolution inputs, it employs the AnyRes strategy, which divides an image into a global contextual patch and multiple high-resolution local patches. This process inherently generates a massive sequence of visual tokens, making it an ideal architecture to validate the efficiency and performance retention of our token reduction approach.

\noindent \textbf{Qwen2.5-VL-7B}~\cite{yang2025qwen25} Qwen2.5-VL is the latest iteration of the Qwen-VL series, featuring significant advancements in comprehending images and videos of various resolutions and aspect ratios. Built upon the Qwen2.5-7B large language model, it incorporates a sophisticated Vision Transformer (ViT) architecture with dynamic resolution support. This allows the model to adaptively process visual inputs into a variable number of tokens based on the image size and complexity, enabling exceptional performance in fine-grained detail perception, OCR, and complex reasoning. Evaluating our method on this model demonstrates its robustness and flexibility across dynamically scaling token sequences.

% \subsection{Visual and Text Token Similarity Analysis}
% Image-to-Image와 Image-to-Text 간 유사도 양상을 비교 분석한다.  
% 실험 결과, Image-to-Image 간 Key 유사도가 Image-to-Text 간 Key 유사도보다 평균적으로 약 2배 높게 나타났다.  
% 이에 따라 본 논문에서는 Image-to-Image 유사도를 기준으로 Image token의 redundancy를 측정하는 전략을 채택한다.

% Having established the Dual Similarity metric, a natural question arises regarding its application scope within Vision-Language Models, which process both image and text tokens. Empirically, we observe that cross-modal pairs (Image-to-Text) exhibit a Dual Similarity approaching zero (mean $\approx 0.002$). This indicates that image and text tokens naturally span independent, orthogonal subspaces. In contrast, intra-modal pairs (Image-to-Image) show significantly high similarity (mean $\approx 0.758$), confirming severe collinearity. Therefore, we can safely restrict our redundancy evaluation exclusively to intra-modal image tokens. This targeted approach explicitly resolves the structural degradation caused by visual token pruning, while preserving cross-modal independence and minimizing computational overhead.

\section{Implementation Details}
\label{supp:imp_details}
For token pruning, we remove tokens at the beginning of each layer (i.e., prior to the attention block). To determine the information magnitude, we first extract the query, key, and value projections for each head to calculate their respective magnitudes. Then we compute the mean information magnitude along the heads. Because this process is independent of the attention computation itself, our proposed method is orthogonal to modern attention acceleration implementations, such as FlashAttention. For both the information magnitude and information duplication measurements, we normalize the results to a range of $[0, 1]$. For the comparative experiments, we implemented baseline methods using their official public repositories. For methods with dynamic pruning ratios, such as PACT, we tuned the hyperparameters to align the resulting pruning ratios with the target token budget. 

\section{Progressive Chunked Maximal Marginal Relevance}
In Algorithm \ref{alg:parallel_mmr}, we describe the detailed process of progressive chunked MMR.

\begin{figure*}[!t]
  \centering
  \resizebox{\textwidth}{!}{\includegraphics[]{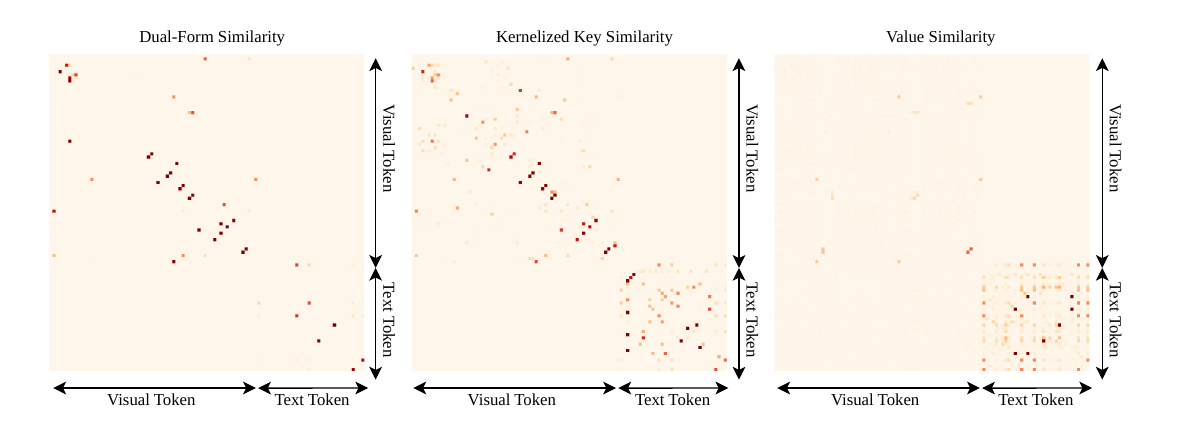}}
  % \caption{index similarity. Diagonal masking for visualization, subset of image token and text token (-200:0)}
  \caption{\textbf{Token similarity visualization in LLaVA-OneVision-7B.} We visualize the similarity between visual and text tokens extracted from Layer 4 across three metrics defined in Eq.~12. For visual clarity, diagonal elements are masked and we use the last 100 tokens of the entire sequence.}
  \label{fig:supp_token_similarity}
\end{figure*}

\section{Further Analyses and Experiments}
\label{supp:further_analyses}
% In this section, we provide further analyses and experiments. 
% We first examine the intra- and inter-similarities of visual and text tokens (\Cref{supp:token sim}.
% We then conduct ablation studies on the two pruning criteria information magnitude and duplication (\Cref{supp:ablation_budget_analysis}) and on the effect of RoPE in each metric (\Cref{supp:ablation_rope}).
% Next, we analyze the effect of pruning layer selection through both performance variation and token-selection consistency across layers (\Cref{supp:pruning_layer_selection}).
% Finally, we assess the scalability of our method on the larger Qwen2.5-VL-32B model (\Cref{supp:scalability}) and ablate the hyperparameters of Progressive Chunked MMR (\Cref{supp:pcmmr_hyperparams}).
In this section, we provide further analyses and experiments.
We first examine the intra- and inter-similarities of visual and text tokens (\Cref{supp:token sim}).
We then conduct ablation studies on the two pruning criteria, information magnitude and duplication (\Cref{supp:ablation_budget_analysis}), and on the effect of RoPE in each metric (\Cref{supp:ablation_rope}).
Next, we analyze the effect of pruning layer selection through both performance variation and token-selection consistency across layers (\Cref{supp:pruning_layer_selection}).
We further assess the scalability of our method on the larger Qwen2.5-VL-32B model (\Cref{supp:scalability}) and ablate the hyperparameters of Progressive Chunked MMR (\Cref{supp:pcmmr_hyperparams}).
Finally, we provide qualitative visualizations to illustrate how our method preserves semantic details in generated descriptions under varying token budgets (\Cref{supp:qualitative}).

\subsection{Visual and Text Token Similarity Analysis}
\label{supp:token sim}

To investigate how visual and text tokens are related, we use three similarity metrics: (1) kernelized key similarity, (2) value similarity, and (3) the square of dual-form similarity. 
According to Eq.~(12) in our manuscript, the dual-form similarity can be factorized into the product of the cosine similarity between kernelized keys in the RKHS and the cosine similarity between value vectors.
We analyze tokens extracted from Layer 4 of LLaVA-OneVision-7B and visualize the last 100 tokens of the sequence for clearer visualization. Results are provided in \Cref{fig:supp_token_similarity}.

Across all three metrics, cross-modality similarity between visual and text tokens remains consistently low, suggesting that these modalities contribute distinct updates to the dual weight space. 
On the other hand, intra-modality patterns vary depending on the metric.
While both modalities exhibit relatively high intra-similarity in the kernelized key space, value similarity reveals that visual tokens are less similar to one another than text tokens.
Notably, under the dual-form similarity, redundancy among text tokens becomes almost negligible, whereas visual tokens form distinct similarity clusters.
These observations suggest that, within the visual token set, there exist redundant tokens that induce highly similar updates to the dual weight. This analysis supports the validity of our token pruning framework, which is designed to selectively remove such redundant visual tokens.

\begin{table*}[!t]
\centering
\setlength{\tabcolsep}{3pt}
\resizebox{\textwidth}{!}{ 
\begin{tabular}{cl cccccccccc |c}
\toprule
\textbf{Budget} & \multicolumn{1}{c}{\textbf{Metric}} & \textbf{AI2D} & \textbf{DocVQA} & \textbf{InfoVQA} & \textbf{MMBench} & \textbf{MME} & \textbf{MMMU} & \textbf{SciQA} & \textbf{TextVQA} & \textbf{MMStar} & \textbf{POPE} & \textbf{Avg. (\%)} \\
\midrule
\multicolumn{2}{c}{Baseline} & 81.3 & 87.1 & 66.0 & 80.7 & 1992.0 & 49.2 & 95.9 & 75.8 & 61.9 & 88.3 & 100.0 \\
\midrule\midrule
\multirow{2}{*}{35.3\%} & Mag. & 79.8 & 85.4 & 61.4 & 79.4 & 2001.5 & 49.0 & 92.0 & 75.6 & 59.1 & 87.6 & 97.8 \\
% \cmidrule{2-13}
 & Mag.+Dup. & 80.5 & 85.3 & 60.6 & 79.6 & 1985.0 & 48.9 & 92.9 & 75.2 & 60.0 & 88.1 & 97.9 \\
\midrule
\multirow{2}{*}{22.2\%} & Mag. & 77.9 & 81.4 & 54.3 & 78.4 & 2003.4 & 48.0 & 91.0 & 73.8 & 56.5 & 85.4 & 94.7 \\
% \cmidrule{2-13}
 & Mag.+Dup. & 79.3 & 82.2 & 56.5 & 78.8 & 1977.9 & 47.8 & 91.5 & 74.4 & 57.3 & 87.3 & 95.6 \\
\midrule
\multirow{2}{*}{11.1\%} & Mag. & 74.7 & 67.6 & 43.2 & 75.9 & 1906.1 & 47.6 & 88.2 & 68.4 & 52.3 & 80.6 & 87.9 \\
% \cmidrule{2-13}
 & Mag.+Dup. & 76.8 & 71.9 & 45.5 & 76.5 & 1947.0 & 47.4 & 89.5 & 70.7 & 54.6 & 84.0 & 90.5 \\
\bottomrule
\end{tabular}
}
\caption{\textbf{Ablation Study on various benchmarks with LLaVA-OneVision-7B.} In all experiments, tokens are pruned after the fourth layer with varying token budgets. Scores are absolute performance, while Avg. (\%) indicates the relative average normalized by the full-token baseline. Budget refers to the token budget. Mag. and Mag.+Dup. refers to pruning with information magnitude only and pruning with both information magnitude and duplication, respectively.}
\label{tab:ablation_budget_analysis}
\end{table*}

\begin{table*}[t]
\centering
% \caption{\textbf{Effect of RoPE in Information Magnitude and Information Duplication metrics.} Magnitude and Duplication columns indicate the inclusion ($\checkmark$) or exclusion ($-$) of Rotary Position Embeddings in the respective computations. Performance across various settings is reported, with the average relative performance (Avg. \%) calculated against the baseline. Scores are rounded for readability.
% }
\resizebox{\linewidth}{!}{
\begin{tabular}{ccc| cccccccccc|c}
\toprule
\textbf{Setting} & \textbf{Magnitude} & \textbf{Duplication} & \textbf{AI2D} & \textbf{DocVQA} & \textbf{InfoVQA} & \textbf{MMBench} & \textbf{MME} & \textbf{MMMU} & \textbf{SciQA} & \textbf{TextVQA} & \textbf{MMStar} & \textbf{POPE} & \textbf{Avg. (\%)} \\
\midrule
(0) & \multicolumn{2}{c}{{Baseline}} & 81.3 & 87.1 & 66.0 & 80.7 & 1992.0 & 49.2 & 95.9 & 75.8 & 61.9 & 88.3 & {100.0} \\
\midrule
(1) & - & \checkmark & 80.5 & 85.3 & 60.6 & 79.6 & 1985.0 & 48.9 & 92.9 & 75.2 & 60.0 & 88.1 & 97.9 \\
(2) & \checkmark& \checkmark & 79.7 & 83.8 & 57.9 & 79.1 & 1990.9 & 48.4 & 93.0 & 74.3 & 59.7 & 88.6 & 97.0 \\
(3) & \checkmark&- & 79.8 & 78.1 & 51.6 & 78.3 & 1903.8 & 44.9 & 93.4 & 72.3 & 58.5 & 88.5 & 94.6 \\
(4) & - & - & 79.7 & 78.8 & 51.2 & 79.0 & 1902.3 & 49.2 & 92.7 & 73.3 & 58.0 & 87.9 & 94.6 \\
\bottomrule
\end{tabular}
}
\caption{\textbf{Effect of RoPE in information magnitude and information duplication metrics.} Magnitude and duplication columns indicate the inclusion ($\checkmark$) or exclusion ($-$) of RoPE in the respective metrics. Reported scores indicate absolute performance on each benchmark. Avg. (\%) represents the average relative performance normalized by the full-token baseline.}
\label{tab:ablation-score-sim}
\end{table*}

\subsection{Ablation Study on Information magnitude and Information Duplication}
\label{supp:ablation_budget_analysis}
% We perform an ablation on the token reduction strategy, comparing "Importance", which refers to pruning with importance scores only, and "Importance + Duplication", which refers to a pruning strategy considering both importance and duplication scores across varying token budgets (35.3\%, 22.2\%, and 11.1\%). 
% The results are reported in \cref{tab:ablation_study}.

% In \cref{sec:method}, we introduce a magnitude-based importance score (\cref{sec:importance}), a similarity metric for quantifying information duplication (\cref{sec:similarity}), and a unified pruning strategy that jointly considers both (\cref{sec:dual-form based token pruning}).
Following the formulation in \Cref{sec:method}, we perform an ablation on the token reduction strategy.
We compare two strategies:
(1) \textbf{Information Magnitude-based pruning} (\Cref{sec:importance}), which selects tokens solely based on the importance score, and
(2) \textbf{Information Magnitude and Duplication-based pruning} (\Cref{sec:dual-form based token pruning}), which additionally incorporates a similarity metric to discourage redundant token selection.
The results are reported across varying token budgets (35.3\%, 22.2\%, and 11.1\%) in \Cref{tab:ablation_budget_analysis}.

For a relatively large token budget (e.g., 35.3\%), the difference between the two strategies is marginal, with both methods achieving nearly 98\% of the baseline performance. However, under an constrained budget (e.g., 11.1\%), the average relative performance of the magnitude-only method degrades to 87.9\%. In contrast, "Magnitude + Duplication" significantly outperforms the importance-only baseline by 2.6 pp.

This performance gap demonstrates that relying solely on information magnitude tends to retain highly informative but redundant tokens. When the budget is sufficient, this redundancy consideration is negligible. However, under strict constraints, it wastes token budgets and leads to the loss of unique information. Accounting for information duplication during token pruning avoids selecting tokens with duplicate information and ensures a more diverse, comprehensive information, thereby effectively mitigating performance degradation at extreme compression rates.

\begin{table*}[!t]
\centering
% \caption{\textbf{Effect of pruning layers on LLaVA-OneVision-7B.} 
% Performance across various benchmarks is reported, with the average relative performance (Avg. \%) calculated against the baseline. Scores are rounded for readability}
\resizebox{\linewidth}{!}{
\begin{tabular}{l cccccccccc|c}
\toprule
\textbf{Method} & \textbf{AI2D} & \textbf{DocVQA} & \textbf{InfoVQA} & \textbf{MMBench} & \textbf{MME} & \textbf{MMMU} & \textbf{SciQA} & \textbf{TextVQA} & \textbf{MMStar} & \textbf{POPE} & \textbf{Avg. (\%)} \\
\midrule
\rowcolor{gray!10}
\textit{Baseline} & 81.3 & 87.1 & 66.0 & 80.7 & 1992.0 & 49.2 & 95.9 & 75.8 & 61.9 & 88.3 & {100.0} \\
\midrule
\rowcolor{blue!5} \multicolumn{12}{l}{\textit{Layers}} \\
0  & 79.5 & 71.0 & 48.8 & 79.8 & 1906.3 & 48.9 & 93.7 & 69.4 & 59.8 & 87.3 & {93.2} \\
2  & 79.6 & 80.2 & 55.0 & 79.2 & 1955.5 & 47.1 & 90.7 & 72.5 & 57.8 & 87.4 & {94.8} \\
4 & 80.4 & 85.3 & 60.6 & 79.5 & 1985.0 & 48.8 & 92.8 & 75.1 & 60.0 & 88.0 & {97.9} \\
6  & 79.9 & 83.6 & 60.4 & 80.4 & 1958.9 & 49.4 & 94.1 & 75.7 & 59.6 & 88.4 & {97.9} \\
8  & 80.2 & 84.9 & 59.5 & 80.1 & 1985.3 & 49.3 & 93.4 & 75.8 & 59.3 & 88.2 & {97.8} \\
10 & 80.3 & 84.5 & 60.6 & 80.4 & 1956.2 & 48.6 & 94.6 & 75.2 & 60.5 & 87.9 & {97.9} \\
\bottomrule
\end{tabular}
}
\caption{\textbf{Effect of pruning layers on LLaVA-OneVision-7B.} Reported scores indicate absolute performance on each benchmark. Avg. (\%) represents the average relative performance normalized by the full-token baseline.}
\label{tab:layer-ablation-results}
\end{table*}

\subsection{Ablation Study of RoPE}
\label{supp:ablation_rope}
Rotary Position Embedding (RoPE) \cite{su2024roformer} is a positional encoding widely used in Large Language Models (LLMs). 
We conduct an ablation study to examine its impact on our two pruning metrics: Information Magnitude (Eq.~10) and Information Duplication (Eq.~12).
Our default pruning configuration (Setting 1) excludes RoPE when computing Information Magnitude, but incorporates it for Information Duplication. 
We compare this against three variants, with results summarized in \Cref{tab:ablation-score-sim}.

We observe two findings from the results. 
First, excluding RoPE from Information Magnitude is beneficial. Including RoPE (Setting 2) degrades performance by 0.9 pp compared to the default setting (Setting 1).
% Given that RoPE influences query–key alignment without changing the intrinsic dual-weight magnitude, excluding RoPE when measuring query–key alignment leads to a more accurate estimation of the information magnitude.
Since RoPE applies a rotation that preserves vector norms, excluding RoPE when measuring query–key alignment leads to a more accurate estimation of the information magnitude.
Second, incorporating RoPE is essential for measuring Information Duplication. 
Removing RoPE from the duplication computation (Setting 4) leads to a substantial performance drop by 3.3 pp relative to Setting 1.
This indicates that identifying redundant tokens requires considering not only semantic similarity but also the relative positions of tokens in the sequence.
Specifically, we argue that incorporating positional information reflects the inductive bias that tokens at different locations have distinct information.

\begin{table*}[!t]
\centering
% \caption{\textbf{Performance comparison of different pruning methods on Qwen2.5-VL-32B.}
% Scores indicate absolute performance, while Avg. (\%) represents the relative average normalized by the full-token baseline.
% }
\resizebox{\linewidth}{!}{
\begin{tabular}{l cccccccccc |c}
\toprule
\textbf{Method} & \textbf{AI2D} & \textbf{DocVQA} & \textbf{InfoVQA} & \textbf{MMBench} & \textbf{MME} & \textbf{MMMU} & \textbf{SciQA} & \textbf{TextVQA} & \textbf{MMStar} & \textbf{POPE} & \textbf{Avg.(\%)} \\
\midrule

\rowcolor{gray!10}
\multicolumn{12}{c}{\textit{Upper Bound, All Tokens} ($\mathbf{100\%}$)} \\

Qwen2.5-VL-32B 
% & 82.4 & 94.6 & 80.4 & 84.2 & 2310.2 & 50.7 & 88.3 & 82.8 & 62.3 & 87.6 & 100.0 \\
& 84.5 & 93.0 & 80.6 & 86.6 & 2433.9 & 58.7 & 91.4 & 77.4 & 64.6 & 84.7 & 100.0 \\

\rowcolor{gray!10}
\multicolumn{12}{c}{\textit{Token Budget: 35.3\%} \textcolor{Green}{($\downarrow\mathbf{64.7\%}$)}} \\

FastV
& 78.6 & 83.8 & 62.0 & 80.8 & 2273.3 & 57.1 & 85.0 & 72.8 & 55.8 & 76.4 & 90.8 \\

VisionZip
& 81.9 & 84.5 & 60.4 & 85.1 & 2316.3 & 58.6 & 88.7 & 71.7 & 61.8 & 83.7 & 94.0 \\

DivPrune
& 82.2 & 82.3 & 59.3 & 85.1 & 2360.5 & 57.4 & 88.9 & 73.1 & 59.6 & 82.9 & 93.4 \\

CDPruner
& 78.8 & 71.7 & 51.0 & 82.4 & 2208.9 & 56.4 & 85.2 & 69.1 & 56.3 & 75.9 & 87.2 \\

\rowcolor{gray!40}
\textbf{Ours} 
& 82.2 & 82.1 & 59.2 & 84.2 & 2381.6 & 58.9 & 90.0 & 75.7 & 61.0 & 82.1 & 94.2 \\

\rowcolor{gray!10}
\multicolumn{12}{c}{\textit{Token Budget: 22.2\%} \textcolor{Green}{($\downarrow\mathbf{77.8\%}$)}} \\

FastV 
& 75.3 & 74.3 & 51.9 & 77.6 & 2083.7 & 55.6 & 83.2 & 69.3 & 50.0 & 66.8 & 84.0 \\

VisionZip 
& 78.9 & 67.9 & 44.0 & 81.5 & 2123.9 & 57.8 & 87.6 & 62.9 & 57.0 & 81.0 & 86.2 \\

DivPrune
& 77.9 & 71.7 & 47.2 & 82.7 & 2234.2 & 54.7 & 86.8 & 68.6 & 57.1 & 81.5 & 87.7 \\

CDPruner
& 72.9 & 60.1 & 42.9 & 78.9 & 2079.5 & 56.1 & 83.1 & 63.3 & 51.8 & 71.1 & 81.3 \\

\rowcolor{gray!40}
\textbf{Ours} 
& 79.4 & 73.5 & 50.0 & 81.9 & 2321.2 & 57.7 & 87.6 & 73.7 & 56.4 & 79.9 & 89.6 \\

\rowcolor{gray!10}
\multicolumn{12}{c}{\textit{Token Budget: 11.1\%} \textcolor{Green}{($\downarrow\mathbf{88.9\%}$)}} \\

FastV
& 70.4 & 55.7 & 36.9 & 72.3 & 1805.8 & 54.3 & 80.2 & 61.7 & 44.0 & 48.3 & 73.2 \\

VisionZip
& 73.0 & 44.5 & 32.3 & 75.6 & 1921.9 & 56.8 & 84.5 & 45.6 & 49.1 & 73.2 & 75.1 \\

DivPrune
& 72.4 & 53.0 & 34.6 & 77.0 & 2093.0 & 54.8 & 82.5 & 59.9 & 50.3 & 75.5 & 78.9 \\

CDPruner
& 67.9 & 43.4 & 32.8 & 71.0 & 1879.2 & 54.3 & 79.3 & 52.0 & 44.5 & 63.8 & 71.8 \\

\rowcolor{gray!40}
\textbf{Ours} 
& 74.0 & 58.1 & 38.4 & 77.9 & 2143.4 & 54.6 & 83.5 & 68.6 & 50.6 & 74.5 & 81.5 \\

\bottomrule
\end{tabular}
}
\caption{\textbf{Performance comparison of different pruning methods on Qwen2.5-VL-32B.} Reported scores indicate absolute performance on each benchmark. Avg. (\%) represents the average relative performance normalized by the full-token baseline.}
\label{tab:qwen2.5-vl-32b}
\end{table*}

\subsection{Analysis on Pruning Layer Selection}
\label{supp:pruning_layer_selection}
In LVLMs, token pruning is performed by compressing the visual token set at a specific layer of the LLM, after which the compressed visual tokens are concatenated with the text tokens and passed to the subsequent layers.
Pruning at earlier layers leads to larger reductions in computation and memory usage during the prefilling stage. However, excessively early pruning may discard tokens that become important in later layers, potentially degrading the model's perception and reasoning capabilities. Therefore, identifying the optimal trade-off between efficiency and performance is crucial.
To this end, we analyze both the performance variation across pruning layers and the consistency of the selected tokens across layers.

\subsubsection{Performance Variation across Pruning Layers.}
We measure model performance across various pruning layers. 
\Cref{tab:layer-ablation-results} shows that pruning at Layer~0 and Layer~2 results in significant performance drops, retaining only 93.2\% and 94.8\% of the baseline performance, respectively. 
In contrast, once pruning is applied from Layer~4 onward, the performance remains consistently around 97.8\% - 97.9\%, with no additional degradation.

\begin{table*}[!t]
\centering
% \caption{\textbf{Effect of Progressive Chunked MMR hyperparameter on LLaVA-OneVision-7B.} 
% All scores except MME and MMMU are scaled to 100 for consistency. Avg. (\%) represents the arithmetic mean of the relative performance across all 10 benchmarks compared to the Baseline. Other Hyperparameters are fixed, varying only one factor at a time, with the default configuration indicated in \textbf{bold}.}
\resizebox{\linewidth}{!}{
\begin{tabular}{l cccccccccc|c}
\toprule
\textbf{Method} & \textbf{AI2D} & \textbf{DocVQA} & \textbf{InfoVQA} & \textbf{MMBench} & \textbf{MME} & \textbf{MMMU} & \textbf{SciQA} & \textbf{TextVQA} & \textbf{MMStar} & \textbf{POPE} & \textbf{Avg. (\%)} \\
\midrule
\rowcolor{gray!10}
\textit{Baseline} & 81.3 & 87.1 & 66.0 & 80.7 & 1992.0 & 49.2 & 95.9 & 75.8 & 61.9 & 88.3 & {100.0} \\
\midrule
\rowcolor{blue!5} \multicolumn{12}{l}{\textit{$\lambda$}} \\
1  & 80.3 & 85.0 & 60.2 & 79.4 & 1963.3 & 49.4 & 92.8 & 75.2 & 59.8 & 88.4 & {97.8} \\
\textbf{5} & 80.5 & 85.3 & 60.6 & 79.6 & 1985.0 & 48.9 & 92.9 & 75.2 & 60.0 & 88.1 & {97.9} \\
10  & 79.9 & 83.6 & 59.7 & 80.0 & 1959.6 & 48.6 & 93.3 & 74.6 & 60.4 & 88.4 & {97.4} \\
\midrule
\rowcolor{green!5} \multicolumn{12}{l}{\textit{$g$}} \\
1 & 80.4 & 83.0 & 58.6 & 80.0 & 1980.7 & 48.3 & 94.1 & 74.6 & 60.9 & 87.8 & {97.4} \\
\textbf{2} & 80.5 & 85.3 & 60.6 & 79.6 & 1985.0 & 48.9 & 92.9 & 75.2 & 60.0 & 88.1 & {97.9} \\
4 & 80.4 & 84.2 & 59.5 & 79.6 & 1969.6 & 49.2 & 92.9 & 75.0 & 59.7 & 87.9 & {97.5} \\
8 & 80.3 & 84.7 & 60.0 & 79.7 & 1991.4 & 49.6 & 92.1 & 75.3 & 59.1 & 88.0 & {97.7} \\
\midrule
\rowcolor{orange!5} \multicolumn{12}{l}{\textit{$b_0$}} \\
1  & 80.1 & 83.7 & 59.6 & 80.3 & 1951.3 & 48.0 & 93.2 & 74.8 & 60.5 & 88.5 & {97.4} \\
\textbf{2} & 80.5 & 85.3 & 60.6 & 79.6 & 1985.0 & 48.9 & 92.9 & 75.2 & 60.0 & 88.1 & {97.9} \\
8  & 80.0 & 83.6 & 59.1 & 79.7 & 1946.7 & 48.4 & 93.5 & 74.8 & 60.2 & 88.2 & {97.2} \\
64 & 80.6 & 83.6 & 60.0 & 79.3 & 1962.2 & 48.7 & 92.9 & 75.2 & 60.2 & 88.3 & {97.5} \\
\bottomrule
\end{tabular}
}
\caption{\textbf{Effect of Progressive Chunked MMR hyperparameters on LLaVA-OneVision-7B.} Default configurations are indicated in \textbf{bold}, with other hyperparameters fixed while varying one factor at a time. Reported scores indicate absolute performance on each benchmark. Avg. (\%) represents the average relative performance normalized by the full-token baseline.}
\label{tab:llava-onevision-pruning-relative}
\end{table*}

\subsubsection{Consistency of Selected Tokens across Layers.}
\begin{figure*}[!h]
  \centering
  \resizebox{0.9\textwidth}{!}{\includegraphics[]{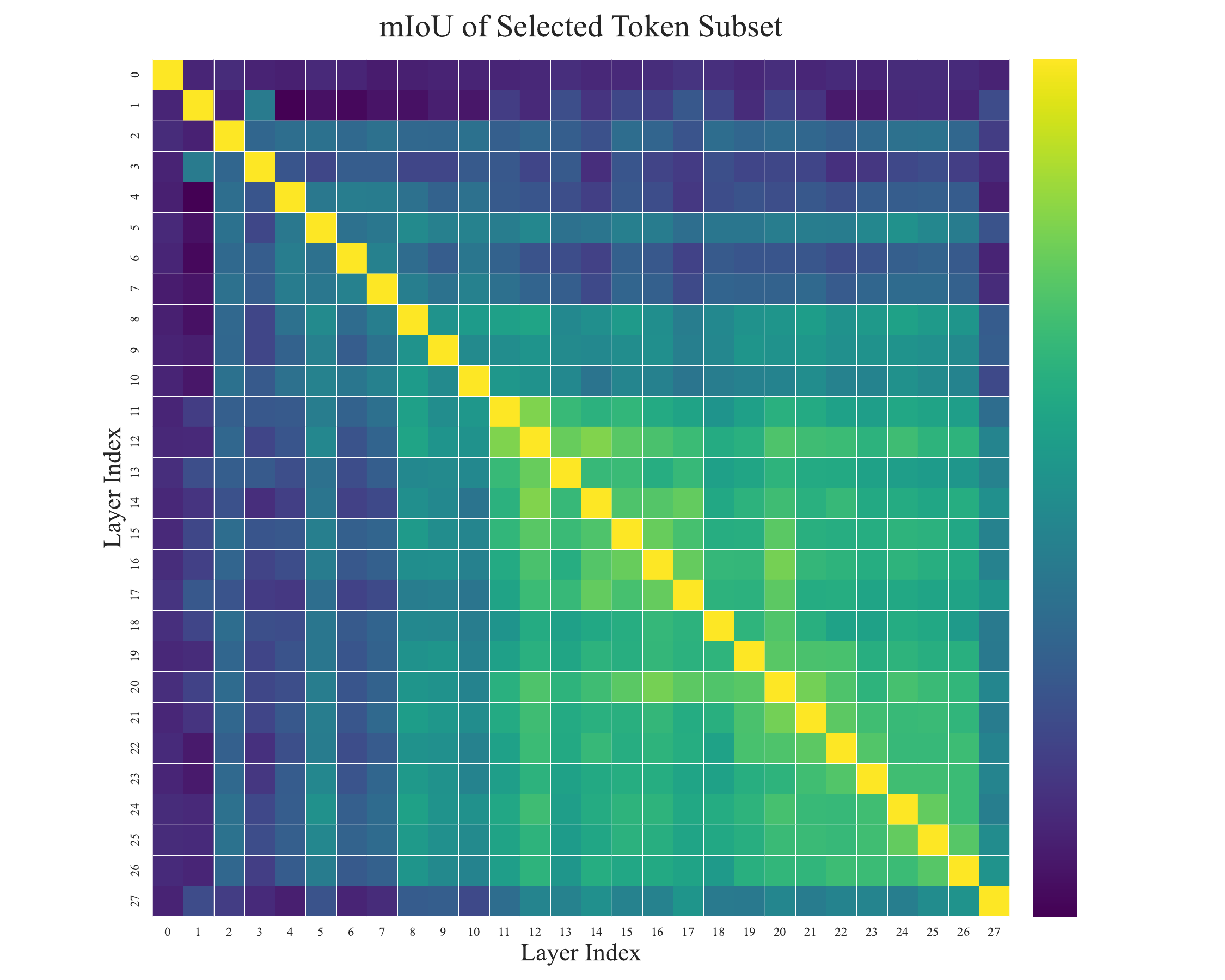}}
  \caption{\textbf{Consistency of selected tokens across layers.} We visualize the mean Intersection over Union (mIoU) of the visual token subsets selected by our pruning method at different transformer layers.}
  \label{fig:supp_important_iou}
\end{figure*}
To examine whether early-layer pruning removes tokens that become important in deeper layers, we measure the mIoU between the token sets selected by our method at different layers with LLaVA-OneVision-7B and ScienceQA dataset.
\Cref{fig:supp_important_iou} shows that token sets selected at very early layers (e.g., Layer~0 and Layer~1) exhibit relatively low similarity to those selected at later layers. In contrast, the sets selected at Layer~4 show consistently high similarity to those of deeper layers.
% Notably, the token sets selected from Layers~4--7 show relatively low similarity to the token set selected at the final layer (Layer~27). 
% We attribute this to the fact that very deep layers are closer to the next-token prediction head, where token importance becomes increasingly specialized for generation.
Notably, similarity decreases at the final layer (Layer 27), which we attribute to the alignment of representations with the next-token prediction objective.

Nevertheless, as shown in the preceding experiment, this discrepancy in similarity does not lead to a meaningful drop in final task performance.

Based on these observations, we select Layer~4 as the optimal pruning layer, as it provides the best balance between computational efficiency and performance.

\subsection{Additional Results: Scalability to Larger Model}
\label{supp:scalability}
To evaluate the scalability of the proposed pruning framework, we conduct additional experiments on the larger Qwen2.5-VL-32B model.
As shown in \Cref{tab:qwen2.5-vl-32b}, our method achieves the best performance across all evaluated token budgets. 
In particular, under more constrained settings with token budgets of 22.2\% and 11.1\%, our approach shows significantly better performance than the baselines.
These results suggest that the proposed dual-form-based token selection criterion remains effective even in larger models.

\subsection{Ablation Study on Progressive Chunked MMR Hyperparameters}
\label{supp:pcmmr_hyperparams}

We perform an ablation on three hyperparameters of Progressive Chunked MMR: the penalty strength $\lambda$, the chunk growth factor $g$, and the initial chunk size $b_0$. The results are reported in \Cref{tab:llava-onevision-pruning-relative}.
For $\lambda$, an excessively large value (e.g., $\lambda=10$) causes over-suppression, where even informative tokens receive overly strong penalties, degrading performance from 97.9\% to 97.4\%. 
For the growth factor $g$, large values cause redundant tokens to be selected before redundancy penalties sufficiently suppress their importance scores.
In contrast, $g=1$ overly emphasizes redundancy reduction, resulting in overly conservative selection.
Empirically, $g=2$ achieves the best performance.
Finally, when the initial chunk size $b_0$ is too large (e.g., 64), the initial selection already contains many redundant tokens, which degrades performance.

% \subsection{Value norm이 잘 안되는 상황에 관한 분석}
% Attention score가 uniform하다고 가정했을때, (혹은 극소수의 token에만 높은 attention score가 있고, 나머지는 noise) 인 상황에서 context length에 비례해서 normalization term의 크기가 커진다는 가설 검증

\subsection{Qualitative Visualization}
\label{supp:qualitative}
To qualitatively analyze the impact of our token pruning method on text generation, we evaluate the model's descriptive capabilities under varying token budgets. We utilize two images from the MME benchmark, prompting the model to provide a detailed description of each image.

As illustrated in \Cref{fig:supp_qualitative}, our approach preserves the core semantic details present in the responses generated by the unpruned vanilla model. Specifically, in \Cref{fig:supp_qualitative_sample1}, we observe that visual tokens corresponding to crucial objects, such as the baseball player and the ball, are preserved after token pruning. By retaining these highly informative tokens, the model consistently generates fine-grained descriptions of the scene, including the uniform, the player's name, and the background context, across all evaluated token budgets.
Similarly, \Cref{fig:supp_qualitative_sample2} demonstrates the model's ability to retain fine-grained details for an image featuring a beer bottle and a glass. 
Our method successfully preserves both the textual and visual information of these objects. 
For instance, specific elements such as the text `since 1890' on the bottle label, the Heineken branding on the glass to the right, and the surrounding background are accurately described regardless of the token budget constraint.

These visualizations show that our method selectively preserves highly informative visual tokens even at lower token budgets.

\subsection{Comparison between PC-MMR and MMR}

We compare exact greedy MMR with Algorithm 1, in respect to accuracy and efficiency.
We considered both of the following two variants of greedy-MMR: (a) Exact, mult. : the same multiplicative penalty, which isolates the effect of chunking, and (b) Exact, add. : the additive form of \Cref{sec:dual-form based token pruning}, which results in $P_i/\max P - \gamma\max_{j\in C}S_{ij}$ with $\gamma{=}0.5$. 
% Below table report averaged accuracy over ScienceQA-IMG, TextVQA, and POPE (Qwen2.5-VL-7B).
% -> table 캡션으로 옮김

\begin{table}[!h]
\centering
\small
\begin{tabular}{lccc}
\toprule
\textbf{Budget} & \textbf{PC-MMR (ours)} & \textbf{Exact, mult.} & \textbf{Exact, add.} \\
\midrule
35.3\% & 84.92 & 85.10 & 84.50 \\
11.1\% & 79.46 & 79.46 & 79.69 \\
\bottomrule
\end{tabular}
\caption{Comparison of PC-MMR with exact multiplicative and additive variants under different token budgets. Average accuracies over ScienceQA-IMG, TextVQA, and POPE (Qwen2.5-VL-7B) are reported.}
\label{tab:exact_comparison}
\end{table}
As shown in \Cref{tab:exact_comparison}, chunking costs at most 0.2pp, and the sign is not even consistent across budgets. 

To validate the efficiency of PC-MMR, we measure each algorithm's latency on single H100: 0.055s per sample for PC-MMR vs. 0.162s for exact greedy on DocVQA at 35.3\% (about 2k visual tokens), and 0.016s vs. 0.027s on ScienceQA (about 0.3K visual tokens). PC-MMR is 1.7–2.9× faster, and the gap grows with because exact greedy needs sequential argmax operations. This is also why \Cref{alg:parallel_mmr} uses a multiplicative penalty: it accumulates across chunks in a streaming fashion with memory per candidate.

\section{Ethical considerations}
\paragraph{License of Existing Assets.}The models used in this work, LLaVA-OneVision-7B and Qwen2.5-VL-7B/32B, are released under the Apache 2.0 license. The evaluation benchmarks are used under their respective licenses: AI2D (CC BY-SA 4.0), TextVQA (CC BY 4.0), ScienceQA (CC BY-NC-SA 4.0), MMMU and MMBench (Apache 2.0), POPE (MIT), NExT-QA (MIT), and EgoSchema (Ego4D License). DocVQA, InfoVQA, MME, MMStar, Video-MME, and MLVU are used under their respective research-use terms. All assets are used in a manner consistent with their intended academic, non-commercial use.

\paragraph{LLM Usage.}
We used AI assistants for grammar checking during the writing process. 
All scientific contributions in this paper are the original work of the authors.

% Consequently, the generated responses maintain their semantic integrity even at lower token budgets.
\begin{figure*}[!h]
% \begin{figure*}[h]
    \centering
    \begin{subfigure}[t]{0.9\textwidth}
        \centering
        \includegraphics[width=\linewidth]{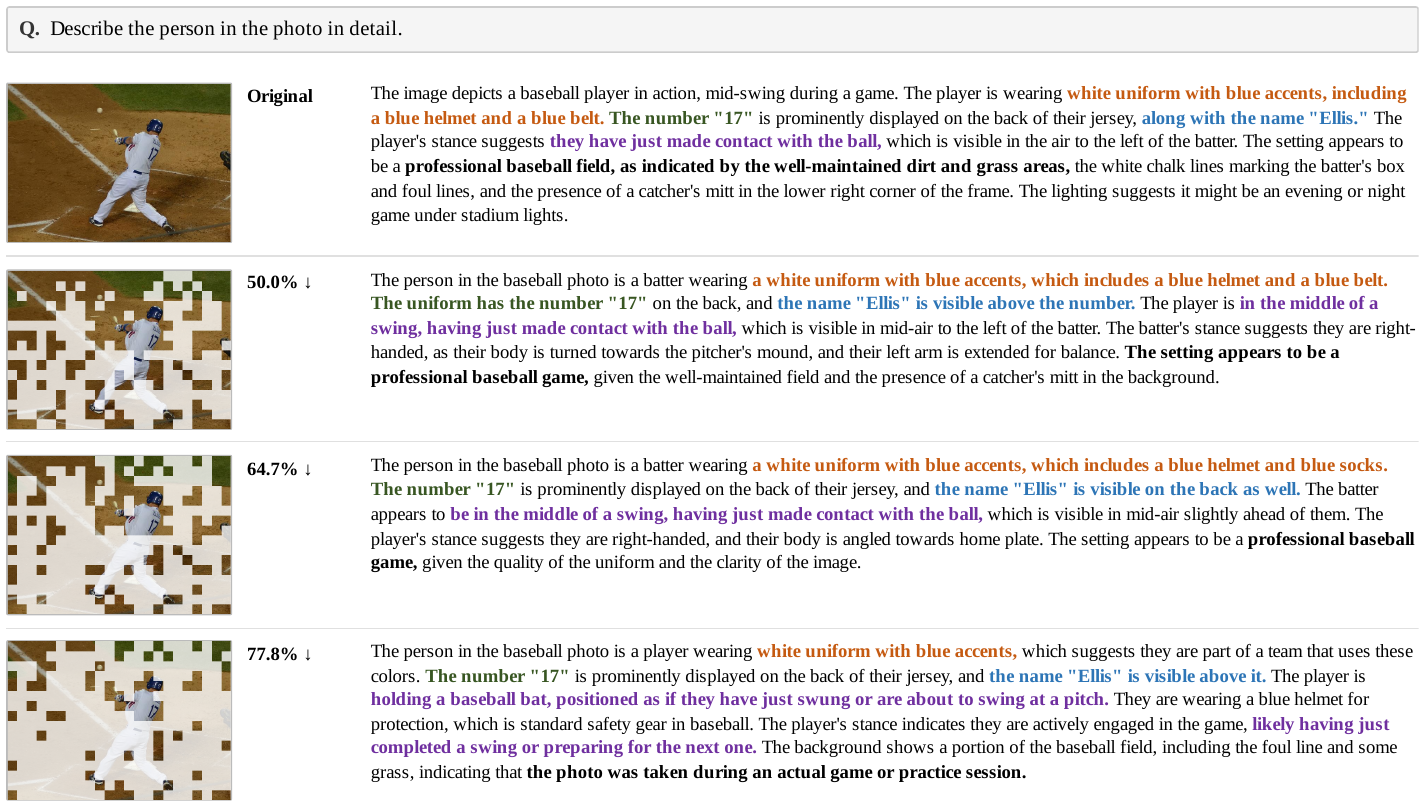}
        \caption{Qualitative results for Sample 1.}
        \label{fig:supp_qualitative_sample1}
    \end{subfigure}
    \vfill
    \begin{subfigure}[t]{0.9\textwidth}
        \centering
        \includegraphics[width=\linewidth]{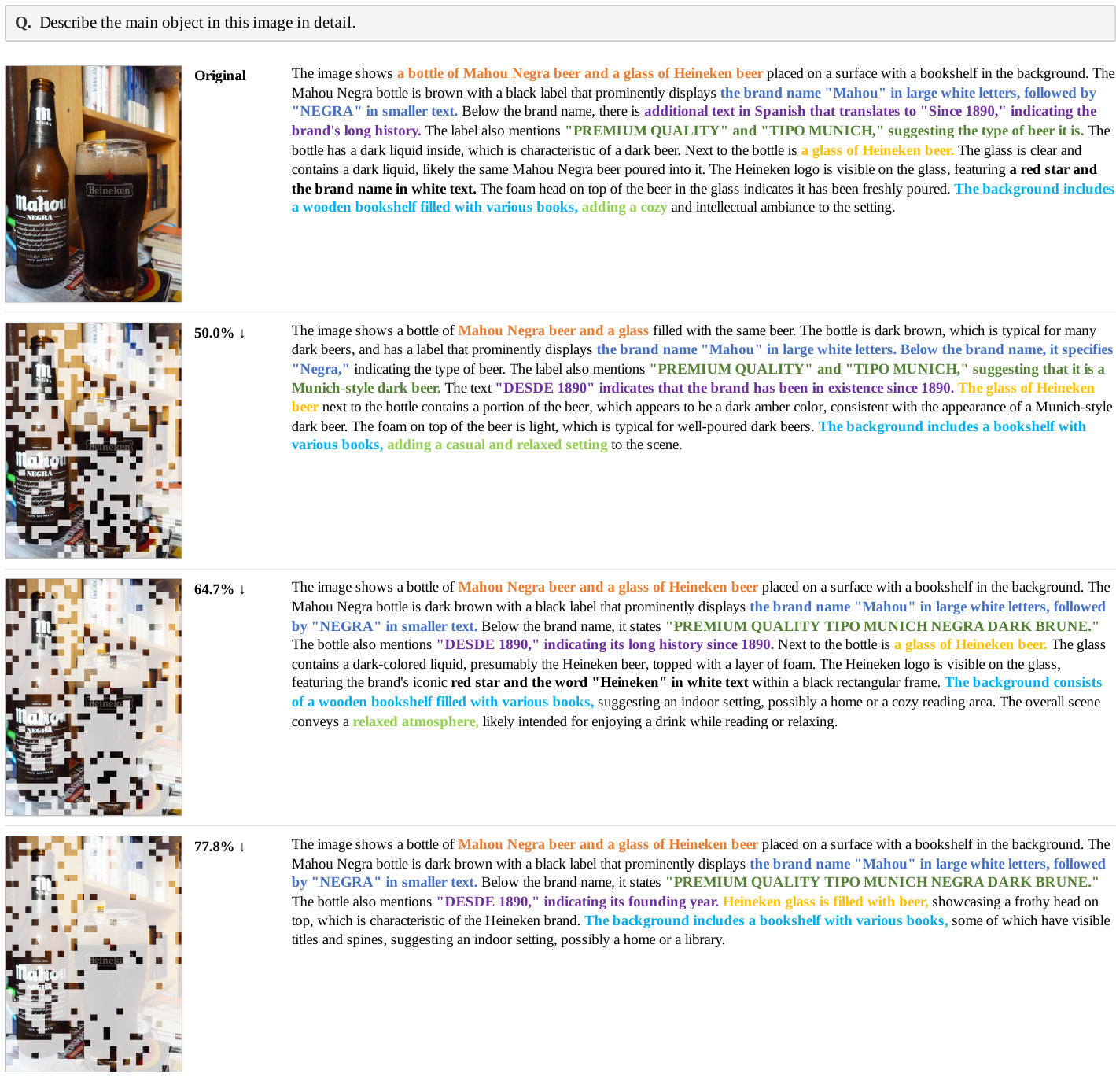}
        \caption{Qualitative results for Sample 2.}
        \label{fig:supp_qualitative_sample2}
    \end{subfigure}
    \caption{
    \textbf{Qualitative visualization of token pruning.} We compare the generated descriptions from the unpruned baseline model with those from our proposed token pruning method under varying token budgets. Corresponding expressions are highlighted in the same color for clarity.
    }
    \label{fig:supp_qualitative}
\end{figure*}

\end{document}